\documentclass[12pt, final]{l4dc2023}

\usepackage{times}
\usepackage{graphicx}
\usepackage{epstopdf}

\usepackage{amsfonts}       
\usepackage{amsmath} 
\usepackage{algorithm}
\usepackage[noend]{algpseudocode}
\usepackage{tcolorbox}
\usepackage{enumitem}
\usepackage{sidecap}
\usepackage{caption}
\usepackage{mathtools}
\usepackage{booktabs}
\usepackage{wrapfig}

\def\relu{\mathrm{ReLU}}
\newcommand{\regionOne}{I}
\newcommand{\regionTwo}{II}
\newcommand{\regionThree}{III}
\newcommand{\regionFour}{IV}
\newcommand{\regionFive}{V}

\newtheorem{problem}{Problem}
\newtheorem*{theorem*}{Theorem}
\newtheorem*{proposition*}{Proposition}

\usepackage{lipsum}

\newcommand\blfootnote[1]{%
  \begingroup
  \renewcommand\thefootnote{}\footnote{#1}%
  \addtocounter{footnote}{-1}%
  \endgroup
}

\author{%
 \Name{Tianqi Cui} \Email{tcui3@jhu.edu}\\
 \addr Department of Chemical and Biomolecular Engineering, Johns Hopkins University, Baltimore, MD, USA
 \AND
 \Name{Thomas Bertalan} \Email{tom@tombertalan.com}\\
 \addr Department of Chemical and Biomolecular Engineering, Johns Hopkins University, Baltimore, MD, USA%
 \AND
 \Name{George Pappas} \Email{pappasg@seas.upenn.edu}\\
 \addr Department of Electrical and Systems Engineering, University of Pennsylvania, Philadelphia, PA, USA%
 \AND
 \Name{Manfred Morari} \Email{
morari@seas.upenn.edu}\\
 \addr Department of Electrical and Systems Engineering, University of Pennsylvania, Philadelphia, PA, USA%
 \AND
 \Name{Yannis Kevrekidis} \Email{yannisk@jhu.edu}\\
 \addr Department of Chemical and Biomolecular Engineering, Johns Hopkins University, Baltimore, MD, USA%
 \AND
 \Name{Mahyar Fazlyab} \Email{mahyarfazlyab@jhu.edu}\\
 \addr Department of Electrical and Computer Engineering, Johns Hopkins University, Baltimore, MD, USA%
}

\begin{document}

\title{Certified Invertibility in Neural Networks via \\ \ Mixed-Integer Programming}

\maketitle

\begin{abstract}
Neural networks are known to be vulnerable to adversarial attacks, which are small, imperceptible perturbations that can significantly alter the network's output. Conversely, there may exist large, meaningful perturbations that do not affect the network's decision (excessive invariance). In our research, we investigate this latter phenomenon in two contexts: (a) discrete-time dynamical system identification, and (b) the calibration of a neural network's output to that of another network.
We examine noninvertibility through the lens of mathematical optimization, where the global solution measures the ``safety" of the network predictions by their distance from the non-invertibility boundary. We formulate mixed-integer programs (MIPs) for ReLU networks and $L_p$ norms ($p=1,2,\infty$) that apply to neural network approximators of dynamical systems. We also discuss how our findings can be useful for invertibility certification in transformations between neural networks, e.g. between different levels of network pruning.
\end{abstract}

\section{Introduction}

Despite achieving high performance in various classification and regression tasks, neural networks do not always guarantee certain desired properties after training. Adversarial robustness is a well-known example, as neural networks can be overly sensitive to carefully designed input perturbations (\cite{szegedy2013intriguing}). This intriguing property also holds in the reverse direction, where neural networks can be excessively insensitive to large perturbations in classification problems. This can cause two semantically different inputs (such as images) to be classified in the same category (\cite{jacobsen2018excessive}).
Indeed, a fundamental trade-off exists between adversarial robustness and excessive invariance (\cite{tramer2020fundamental}), which is mathematically related to the noninvertibility of the input-output map defined by the neural network.

To address the issue of noninvertibility and excessive invariance, one can consider invertible-by-design architectures. Invertible neural networks (INNs) have been used to design generative models (\cite{donahue2019large}), implement memory-saving gradient computation (\cite{gomez2017reversible}), and solve inverse problems (\cite{ardizzone2018analyzing}). However, commonly used INN architectures suffer from exploding inverses. In this paper, we focus on certifying the (possible) non-invertibility of conventional neural networks after training.
We specifically study two relevant invertibility problems: (i) local invertibility of neural networks, where we verify whether a dynamical system parameterized by a neural network is locally invertible around a certain input (or trajectory), and compute the largest region of local invertibility; and (ii) local invertibility of transformations between neural networks, where we certify whether two ``equivalent'' neural networks (e.g. resulting from different levels of pruning) can be transformed (or calibrated) to each other locally via an invertible map.
We develop mathematical tools based on mixed-integer linear/quadratic programming for characterizing non-invertibility, which can be applied to neural network approximators of dynamical systems, as well as transformations between different neural networks.

\paragraph{Related Work} Noninvertibility in neural networks was first studied in the 1990s (\cite{GICQUEL19988,298587}). More recently, several papers have focused on the global invertibility property in neural networks, including works such as \cite{chang2017reversible, teshima2020couplingbased, chen2018neural, 10.5555/3327546.3327578, Jaeger2014ControllingRN}. The invertibility of neural networks has been analyzed (\cite{Behrmann2018AnalysisOI}), and invertible architectures have been developed for applications such as generative modeling (\cite{chen2019residualflows}), inverse problems (\cite{Ardizzone2019AnalyzingIP}), and probabilistic inference (\cite{9298920}).
Some of these networks, such as RevNet (\cite{gomez2017reversible}), NICE (\cite{dinh2015nice}), and real NVP (\cite{dinh2017density}), partition the input domains and use affine or coupling transformations as the forward pass, resulting in nonzero determinants and keeping the Jacobians (block-)triangular with nonzero diagonal elements. Others, like i-ResNet (\cite{behrmann2019invertible}), have no analytical forms for the inverse dynamics, yet their finite bi-Lipschitz constants can be derived. Both methods can guarantee global invertibility. A comprehensive analysis of these architectures can be found in \cite{behrmann2020understanding, song2019mintnet}. However, a theoretical understanding of the expressiveness of these architectures, as well as their universal approximation properties, is still incomplete. Compared to standard networks like multi-layer perceptrons (MLPs) or convolutional neural networks (CNNs), invertible neural networks (INNs) are computationally demanding.
Neural ODE (\cite{chen2018neural}) uses an alternative method to compute gradients for backward propagation, while i-ResNet (\cite{behrmann2019invertible}) has restrictions on the norm of every weight matrix to be enforced during the training process. In most cases, the input domain of interest is a small subset of the whole space. For example, the grey-scale image domain in computer vision problems is $[0, 1]^{H \times W}$, where $H$ and $W$ are the height and width of the images; it is unnecessary to consider the entire$\mathbb{R}^{H \times W}$. We thus focus on {\em local invertibility}: how do we determine if our network is invertible on a given domain, and if not, how do we quantify noninvertibility?

\section{Invertibility Certification of Neural Networks and of Transformations between them}

Here we pose the verification of local invertibility of continuous functions as optimization problems. We then show that for  ReLU networks, this leads to a mixed-integer linear/quadratic program.  For an integer $q \geq 1$, we denote the $L_q$-ball centered at $x_c$ by $\mathcal{B}_q(x_c,r) = \{x \in \mathbb{R}^n \mid \|x-x_c\|_q \leq r\}$ (the notation also holds when $q \rightarrow +\infty$).

\subsection{Invertibility Certification of ReLU Networks via Mixed-Integer Programming} \label{MILP1}

\begin{problem}[Local Invertibility of NNs] \label{problem 1}
	%
	Given a neural network $f: \mathbb{R}^m \mapsto \mathbb{R}^m$ and a point $x_c \in \mathbb{R}^m$ in the input space, we want to find the largest radius $r> 0$ such that $f$ is invertible on $\mathcal{B}_q(x_c,r)$, i.e., $f(x) \neq f(y)$ for all $x,y \in \mathcal{B}_q(x_c,r)$, $x \neq y$. \footnote{Here$f$ has the same domain/co-domain dimension. Our mixed-integer formulation does not require this assumption.}
\end{problem}

\begin{wrapfigure}{r}{0.3 \textwidth}
  \begin{center}
    \includegraphics[width=0.3 \textwidth]{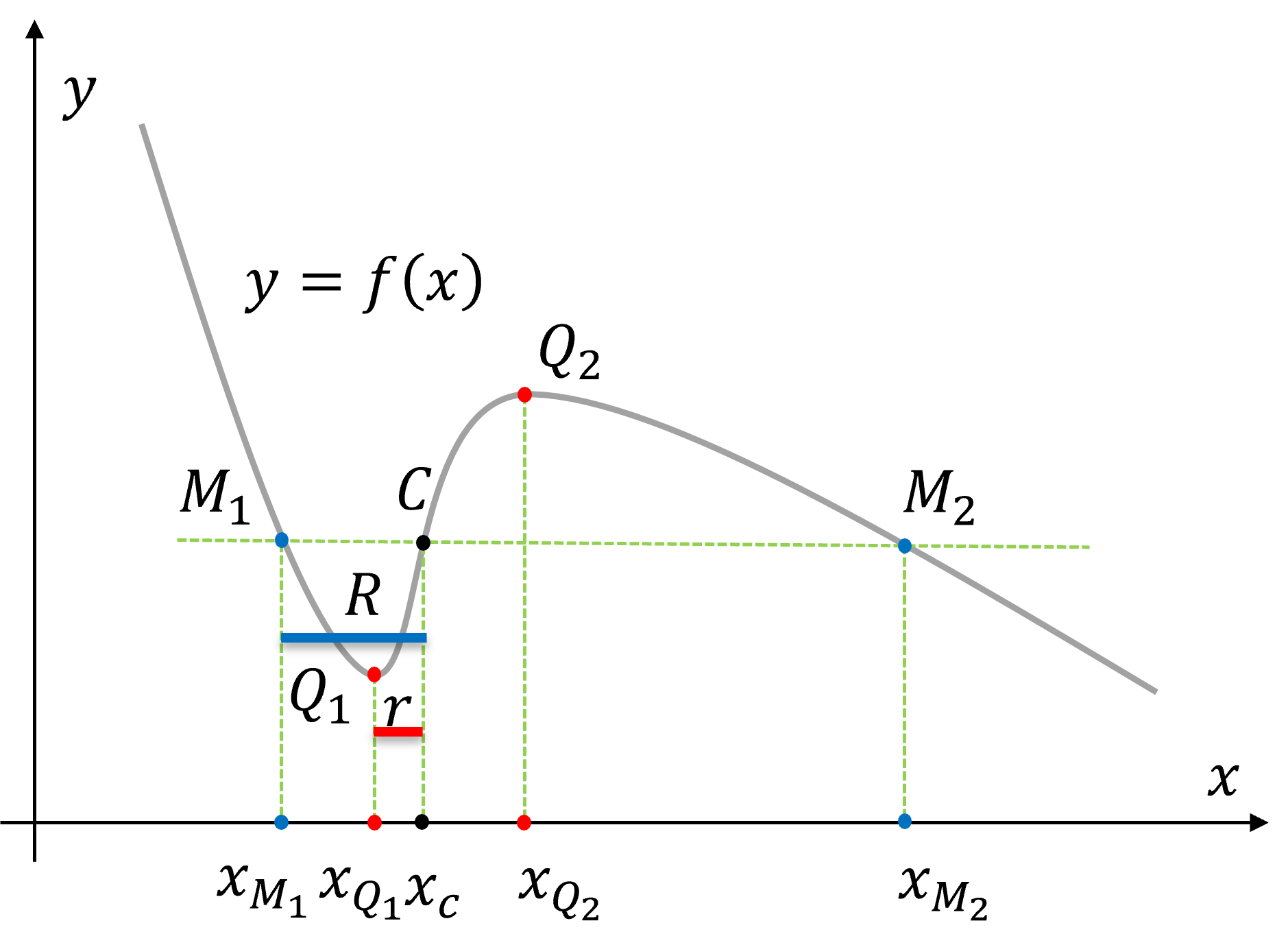}
  \end{center}
  \caption{\small{Illustration of problems 1 and 2 in one dimension.}}
  \label{fig: 1D sketch}
\end{wrapfigure}

Another relevant problem is to verify whether, for a particular point, a nearby point exists with the same forward image. 
We formally state the problem as follows.


\begin{problem}[Pseudo Local Invertibility of NNs] \label{problem 2}
	Given a neural network $f: \mathbb{R}^m \mapsto \mathbb{R}^m$ and a point $x_c \in \mathbb{R}^m$ in the input space, we want to find the largest radius $R> 0$ such that $f(x) \neq f(x_c)$ for all $x \in \mathcal{B}_q(x_c,R)$, $x \neq x_c$.
\end{problem}




If $r$ and $R$ are the optimal radii in Problems \ref{problem 1} and \ref{problem 2} respectively, we must have $r \leq R$.
%
%
%
%
%
For Problem \ref{problem 1}, the ball $\mathcal{B}_q(x_c,r)$ just ``touches'' the $J_0$ set (i.e. the set of points where $f'=0$); for Problem \ref{problem 2}, the ball $\mathcal{B}_q(x_c,R)$ extends to the ``other'' closest preimage of $f(x_c)$. Figure \ref{fig: 1D sketch} illustrates both concepts in the one-dimensional case. For the scalar function $y = f(x)$ and around a particular input $x_c$, we show regions with local invertibility and pseudo invertibility. The points $Q_1= (x_{Q_1}, y_{Q_1})$ and $Q_2= (x_{Q_2}, y_{Q_2})$ are two closest turning points (elements of the $J_0$ set) to the point $C =(x_c, y_c)$; $f$ is uniquely invertible (bi-Lipschitz) on the open interval $(x_{Q_1}, x_{Q_2})$, so that the optimal solution to Problem \ref{problem 1} is: $r = \min \{|x_{Q_1} - x_c|, |x_{Q_2} - x_c|\} = |x_{Q_1} - x_c|$. 
Noting that  $M_1 = (x_{M_1}, y_{M_1})$ and $M_2 = (x_{M_2}, y_{M_2})$ are two closest points that have the same $y$-coordinate as the point $C = (x_c, y_c)$, the optimal solution to Problem \ref{problem 2} is $R = \min \{|x_{M_1} - x_c|, |x_{M_2} - x_c|\} = |x_{M_1} - x_c|$.

%
We now state our first result, posing the local invertibility of a function (such as a neural network) as a constrained optimization problem.
\begin{theorem}[Local Invertibility of Continuous Functions] \label{theorem: local_inv}
	Let $f \colon \mathbb{R}^m \to \mathbb{R}^m$ be a continuous function and $\mathcal{B} \subset \mathbb{R}^m$ be a compact set. Consider the following optimization problem,
\begin{alignat}{2} \label{opt problem 1}
p^\star \leftarrow& \mathrm{max} \quad  && \|x-y\| \quad \text{subject to } x,y \in \mathcal{B}, \quad f(x)=f(y).
\end{alignat}
Then $f$ is invertible on $\mathcal{B}$ if and only if $p^\star =0$.
\end{theorem}
%

\begin{theorem}[Pseudo Local Invertibility]  \label{theorem: local_pseudo}
	Let $f \colon \mathbb{R}^m \to \mathbb{R}^m$ be a continuous function and $\mathcal{B} \subset \mathbb{R}^m$ be a compact set. Suppose $x_c \in \mathcal{B}$. Consider the following optimization problem,
	\begin{align} \label{opt problem 2}
		P^\star \leftarrow  \mathrm{max} \quad  \|x-x_c\| \quad
		\text{subject to } x \in \mathcal{B},  \quad f(x)=f(x_c).
	\end{align}
	Then we have $f(x) \neq f(x_c)$ for all $x \in \mathcal{B} \setminus \{x_c\}$ if and only if $P^\star =0$.
\end{theorem}
%
Note that by adding the equality constraint $y=x_c$ to Problem \eqref{opt problem 1}, we obtain Problem \eqref{opt problem 2}. Hence, we will only focus on Problem \eqref{opt problem 1} in the sequel.

\paragraph{Mixed-Integer Formulation of Problem \eqref{opt problem 1}} We now show that for a given ball $\mathcal{B}_{\infty}(x_c,r)$ in the input space, and  piecewise linear networks with ReLU activations, the optimization problem in \eqref{opt problem 1} can be cast as an MILP. 
%
%
We start by noting that a single ReLU constraint $y = \max(0,x)$ with pre-activation bounds  $\underline{x} \leq x \leq \bar{x}$ can be equivalently described by the following mixed-integer linear constraints (\cite{tjeng2017evaluating}),
\begin{align}
	y=\max(0,x), \ \underline{x} \leq x \leq \bar{x} \iff \{y \geq 0, \ y \geq x, y \leq x - \underline{x} (1-t), \ y \leq \bar{x} t, \ t \in \{0,1\}\}, \label{eq: relu_decomp}
\end{align}
where the binary variable $t \in \{0,1\}$ is an indicator of the activation function being active ($y=x$) or inactive ($y=0$). %
%
%
Now consider an $\ell$-layer feed-forward fully-connected ReLU network, 
\begin{align} \label{eq: nn equations}
x^{(k+1)} = \max(W^{(k)} x^{(k)} + b^{(k)},0) \text{ for } k=0,\cdots,\ell-1; \
f(x^{(0)}) = W^{(\ell)} x^{(\ell)} + b^{(\ell)},
\end{align}
%
%
where $x^{(k)} \in \mathbb{R}^{n_k}$ ($n_0=m$),  $W^{(k)} \in \mathbb{R}^{n_{k+1} \times n_k},b^{(k)} \in \mathbb{R}^{n_{k+1}}$ are the weight matrices and bias vectors of the affine layers. We denote $n = \sum_{k=1}^{\ell} n_{k}$ the total number of neurons.  
%
%
%
%
Suppose $l^{(k)}$ and $u^{(k)}$ are known elementwise lower and upper bounds on the input to the $(\ell+1)$-th activation layer, i.e., $l^{(k)} \leq W^{(k)} x^{(k)} +b^{(k)}\leq u^{(k)}$. Then the neural network equations are equivalent to a set of mixed-integer constraints as follows,
\begin{align} \label{eq:MIL_NN}
	x^{(k+1)} \!=\! \max(W^{(k)} x^{(k)} + b^{(k)},0) \Leftrightarrow \begin{cases}
		x^{(k+1)} \geq W^{(k)} x^{(k)} + b^{(k)} \\ 
		x^{(k+1)} \leq W^{(k)} x^{(k)} + b^{(k)} - l^{(k)} \odot (\mathrm{1}_{n_{k+1}}-t^{(k)}) \\
		x^{(k+1)} \geq 0, \quad x^{(k+1)} \leq u^{(k)} \odot t^{(k)},
	\end{cases} 
\end{align}
where $t^{(k)} \in \{0,1\}^{n_{k+1}}$ is a vector of binary variables for the $(k+1)$-th activation layer and $\mathrm{1}_{n_{k+1}}$ denotes vector of all $1$'s in $\mathbb{R}^{n_{k+1}}$. We note that the element-wise pre-activation bounds $\{ l^{(k)} , u^{(k)} \}$ can be precomputed by, for example, interval bound propagation or linear programming, assuming known bounds on the input of the neural network (\cite{weng2018towards, zhang2018efficient,hein2017formal,wang2018efficient,wong2018provable}). 
Since the state-of-the-art solvers for mixed-integer programming are based on branch $\&$ bound algorithms (\cite{bandb, 10.5555/247975}), tight pre-activation bounds will allow the algorithm to prune branches more efficiently and reduce the total running time.

\begin{align} \label{eq: local invertibility MIP}
p^\star & \leftarrow \mathrm{max} \ w 
\text{ subject to }  \|x^{(0)}-x_c\|_{\infty} \leq r, \notag \ \ \text{} \|y^{(0)}-x_c\|_{\infty} \leq r \\ 
\mathrm{(I)}: & \begin{cases}
(x^{(0)}-y^{(0)}) \leq  w \mathrm{1}_{n_0} \leq (x^{(0)}-y^{(0)}) + 4r(\mathrm{1}_{n_0}-F) \\
-(x^{(0)}-y^{(0)}) \leq  w \mathrm{1}_{n_0} \leq -(x^{(0)}-y^{(0)}) + 4r(\mathrm{1}_{n_0}-F') \\
F + F' \leq \mathrm{1}_{n_0}, \mathrm{1}_{n_0}^\top (F+F') =1, 
 F,F' \in \{0,1\}^{n_0}
\end{cases} \notag \\ 
\mathrm{(II)}: & \ W^{(\ell)} x^{(\ell)} = W^{(\ell)} y^{(\ell)} \\ \notag
& \text{for } k=0,\cdots,\ell-1: \\ \notag
\mathrm{(III)}: & 
\begin{cases}
    x^{(k+1)} \geq W^{(k)} x^{(k)} + b^{(k)}, y^{(k+1)} \geq W^{(k)} y^{(k)} + b^{(k)} \\ 
    x^{(k+1)} \leq W^{(k)} x^{(k)} + b^{(k)} - l^{(k)} \odot (1-t^{(k)}), y^{(k+1)} \leq W^{(k)} y^{(k)} + b^{(k)} - l^{(k)} \odot (1-t^{(k)}) \\
    x^{(k+1)} \geq 0, y^{(k+1)} \geq 0,
    x^{(k+1)} \leq u^{(k)} \odot t^{(k)}, y^{(k+1)} \leq u^{(k)} \odot t^{(k)}; t^{(k)},s^{(k)} \in \{0,1\}^{n_k + 1},
\end{cases} \notag 
\end{align} 
Having represented the neural network equations by mixed-integer constraints, it remains to encode the objective function $\|x^{(0)}-y^{(0)}\|$ as well as the set $\mathcal{B}$.
We assume that $\mathcal{B}$ is an $L_\infty$ ball around a given point $x_c$, i.e., $\mathcal{B} = \mathcal{B}_{\infty}(x_c,r)$. Furthermore, for the sake of space, we only consider $L_\infty$ norms for the objective function.  Specifically, consider the equality $w = \|x^{(0)}-y^{(0)}\|_{\infty}$. This equality can be encoded as mixed-integer linear constraints by introducing $2n_0$ mutually exclusive indicator vectors($F$ and $F'$ each with $n_0$ coordinates). This would lead to the MILP in \eqref{eq: local invertibility MIP}, where the set of constraints in $\mathrm{(I)}$ model the objective function $\|x^{(0)}-y^{(0)}\|_{\infty}$, and the set of constraints $\mathrm{(III)}$ encodes $x^{(k+1)}=\max(W^{(k)} x^{(k)}+b^{(k)},0)$ and $y^{(k+1)}=\max(W^{(k)} y^{(k)}+b^{(k)},0)$ which is exactly \eqref{eq:MIL_NN}. The constraint $\mathrm{(II)}$ enforces $f(x^{(0)}) = f(y^{(0)})$ which can be inferred from \eqref{eq: nn equations}.
To see the correctness of $\mathrm{(I)}$, suppose $F_j=1$ for some $j = 1, \cdots, n_0$. Then, we must have $F'_i = 0$ for $\forall i = 1, \cdots, n_0$ and $F_i = 0$ for $\forall i \neq j$. This implies $w = (x_j^{(0)}-y_j^{(0)}) \geq (x_i^{(0)}-y_i^{(0)})$ for $\forall i\neq j$, and $w \geq -(x_i^{(0)}-y_i^{(0)})$ for $\forall i$. A similar argument can be made when $F'_j=1$ for some $j=1,\cdots,n_0$.
The optimization problem \eqref{eq: local invertibility MIP} has a total of $2(n_0+n)$ integer variables.

%
\begin{remark}
    Using the $\ell_2$ norm for both the objective function and the ball $\mathcal{B}_2(x_c,r)$, leads to a mixed-integer quadratic program (MIQP). However, \eqref{eq: local invertibility MIP} remains an MILP in the $\ell_1$ norm case.
\end{remark}





\subparagraph{Largest Region of Invertibility (Problem \ref{problem 1})}For a fixed radius $r \geq 0$, the optimization problem \eqref{eq: local invertibility MIP} either verifies whether $f$ is invertible on $\mathcal{B}_{\infty}(x_c,r)$ or it finds counter examples $x^{(0)} \neq y^{(0)}$ such that $f(x^{(0)})=f(y^{(0)})$. Thus, we can find the maximal $r$ by performing a bisection search on $r$. 




To close this section, we consider the problem of invertibility certification in transformations between two functions (and in particular neural networks).

\subsection{Invertibility Certification of Transformations between Neural Networks} \label{MILP2}

Training two neural networks for the same regression or classification task practically never gives identical networks.
%
Numerous criteria exist for comparing the performance of different models (e.g. accuracy in classification, or mean-squared loss in regression). Here we explore whether two different models {\em can be calibrated to each other} (leading to a {\em de facto} implicit function problem). 
Extending our analysis provides invertibility 
guarantees for the transformation from output of network 1 to output of network 2.

\begin{problem}[Transformation Invertibility] \label{problem 3}
	Given two functions $f_1,f_2 \colon \mathbb{R}^m \to \mathbb{R}^m$ (e.g. two neural networks) and a particular point $x_c \in \mathbb{R}^m$ in the input space, we would like to find the largest ball $\mathcal{B}_q(x_c,r)$ over which $f_2$ is a function of $f_1$.
\end{problem}

\begin{theorem}  \label{theorem: trans}
	Let $f_1 \colon \mathbb{R}^m \to \mathbb{R}^n$, $f_2 \colon \mathbb{R}^m \to \mathbb{R}^n$ be two continuous functions and $\mathcal{B} \subset \mathbb{R}^m$ be a compact set. 
	Then  $f_2$ is a function of $f_1$ on $\mathcal{B}$ if and only if $p_{12}^\star = 0$, where
    \begin{align} \label{opt problem 3}
		p_{12}^\star \leftarrow  \mathrm{max} \quad  \|f_2(x^{(1)}) - f_2(x^{(2)})\| \quad
		\text{subject to } x^{(1)}, x^{(2)} \in \mathcal{B},  \quad f_1(x^{(1)}) = f_1(x^{(2)}).
	\end{align}
\end{theorem}




\noindent Similar to Problem \ref{problem 1}, we can pose Problem \ref{problem 3} as a mixed-integer program. Furthermore, we can also define $p_{21}^\star$, whose zero value verifies whether $f_1$ is a function of $f_2$ over $ \mathcal{B}$. It is straightforward that $p_{12}^\star = p_{21}^\star = 0$ if and only if $f_2$ is an invertible function of $f_1$.

\section{Local Invertibility of Dynamical Systems and Neural Networks}

Noninvertibility can lead to catastrophic consequences not only in classification but also in regression, particularly in dynamical systems prediction. The flow of smooth differential equations is invertible when it exists, yet traditional numerical integrators used to approximate them can be noninvertible. Neural network approximations of the corresponding map also suffer from this potential pathology. Here, we study non-invertibility in the context of dynamical systems predictions.


Continuous-time dynamical systems, in particular autonomous ordinary differential equations (ODEs) have the form $dX(t) / dt = f(X(t)), X(t = t_0) = X_0$,
where $X(t) \in \mathbb{R}^m$ are the state variables of interest; 
$f: \mathbb{R}^m \mapsto \mathbb{R}^m$ relates the states to their time derivatives; $X_0 \in \mathbb{R}^m$ is the initial condition at $t_0$. 
If $f$ is uniformly Lipschitz continuous in $X$ 
and continuous in $t$, 
the Cauchy-Lipschitz theorem provides the existence and uniqueness of the solution. 
%

In practice, we observe the states $X(t)$ at discrete points in time, starting at $t_0 = 0$. 
For 
a fixed timestep $\tau \in \mathbb{R}^+$, and  $\forall n \in \mathbb{N}$, 
$t_n = n \tau$ denotes the $n$-th time stamp, 
and $X_n = X(t = t_n)$ the corresponding state values. 
Now we will have:
\begin{equation}
    X_{n + 1} := F(X_n) = X_n + \int_{t_n}^{t_{n + 1}} f(X(t)) dt; \ X_n = F^{-1} (X_{n + 1}). \label{odeint}
\end{equation}
This equation also works as the starting point of many numerical ODE solvers. 

For the time-one map in \eqref{odeint}, 
the inverse function theorem provides a sufficient condition for its invertibility:
%
If $F$ is a continuously differentiable function from an open set $\mathcal{B}$ of $\mathbb{R}^m$ into $\mathbb{R}^m$, 
and the Jacobian determinant of $F$ at $p$ is nonzero,
then $F$ is invertible near $p$. Thus, 
%
%
if we define the {\em noninvertibility locus} as the set $J_0(F) = \{p \in \mathcal{B}:$ $  \det(\mathbf{J}_F(p)) = 0 \}$; then the condition $J_0(F) = \emptyset$ 
guarantees global invertibility of $F$ (notice that this condition is not necessary: the scalar function $F(X) = X^3$ provides a counterexample). 
%
If $F$ is continuous over $\mathcal{B}$ but not everywhere differentiable,
then the definition of $J_0$ set should be altered to:
\begin{equation}
J_0(F) = \left\{p \in \mathcal{B}: \forall N_0(p), \exists\, p_1, p_2 \in N_0(p), p_1 \neq p_2, \text{ s.t. } \det(\mathbf{J}_F(p_1)) \det(\mathbf{J}_F(p_2)) \leq 0 \right\}. \label{eq: J0_def_ext}
\end{equation}

\subparagraph{Numerical Integrators are (often) Noninvertible}
Numerically approximating the integral in \eqref{odeint} can introduce noninvertibility in the transformation. A simple one-dimensional illustrative ODE example is $f(X) \!=\! X^2 + bX + c, \ X(t = 0) = X_0$, where $b, c \in \mathbb{R}$ are two fixed parameters. 
Although the analytical solution \eqref{odeint} is invertible, a forward-Euler discretization with step $\tau$ gives
\begin{equation}
    X_{n + 1} = F(X_n) = X_n + \tau(X_n^2 + bX_n + c) \Rightarrow
    \tau X_n^2 + (\tau b + 1) X_n + (\tau c - X_{n + 1}) = 0. \label{euler_1d_inv}
\end{equation}
Given a fixed $X_{n + 1}$, Equation \eqref{euler_1d_inv} is quadratic w.r.t. $X_n$; this 
determines the local invertibility of $F$ based on $\Delta = (\tau b + 1)^2 - 4 \tau (\tau c - X_{n + 1})$: no real root if $\Delta < 0$; one real root with multiplicity 2 
if $\Delta = 0$; and two distinct real roots
if $\Delta > 0$. In practice, one uses small timesteps $\tau \ll 1$ for accuracy/stability,
leading to the last case: there will always exist a solution $X_n$ close to $X_{n + 1}$, and a second preimage, far away from the region of our interest, and arguably physically irrelevant (to $X_n \rightarrow -\infty$ as $\tau \rightarrow 0$).
On the other hand, as $\tau$ grows, 
the two roots move closer to each other, 
$J_0(F)$ moves close to the regime of our simulations, and noninvertibility can have visible implications on the predicted dynamics. 
Thus, choosing a small timestep in explicit integrators guarantees desirable accuracy, and simultaneously {\em practically} mitigates noninvertibility pathologies in the dynamics.

\section{Numerical Experiments}


We now present experiments with $\relu$  multi-layer perceptrons (MLPs) in regression problems, and also transformations between two $\relu$ networks.
To solve the Mixed-integer programs we use \cite{gurobi}. To find the pre-activation bounds, we use interval bound propagation.


\paragraph{1D Example}  
We use a 1-10-10-1 randomly generated fully-connected neural network $f$ with $\relu$ activations. We find the largest interval around the points $x=-1.8,-1,-0.3$ on which $f$ is invertible (Problem \ref{problem 1}), and the largest interval around the point $x=-1$ on which any other points inside the region will not map to $f(-1)$ (Problem \ref{problem 2}). The results are plotted in the inset of Figure \ref{fig: 1D example}, where intervals in red and blue respectively represent the optimal solutions for the two problems. The computed largest certified radii are 0.157, 0.322, 0.214, and 0.553.

\begin{figure}[H]
	\centering
	\includegraphics[width=0.3\linewidth]{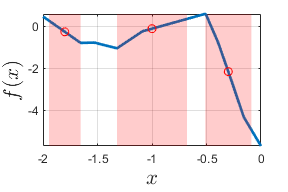}
	\includegraphics[width=0.3\linewidth]{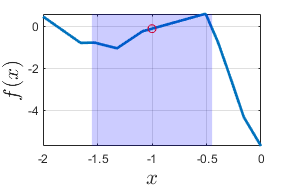}
	\caption{\small{
        Solutions to
        Problem \ref{problem 1} (left, red)
        and Problem \ref{problem 2} (right, blue)
        for the MLP corresponding to a randomly-generated ReLU network (see text).
    }}
	\label{fig: 1D example}
\end{figure}

\paragraph{2D Example: the Brusselator Model} The Brusselator (\cite{doi:10.1063/1.1679748}) is a two-variable $(x, y)$ ODE system  depending on parameters $(a, b)$, that describes oscillatory dynamics in a theoretical chemical reaction scheme. We use its forward-Euler discretization 
\begin{equation}  
    x_{n + 1} = x_n + \tau (a + x_n^2 y_n - (b + 1) x_n), \ y_{n + 1} = y_n + \tau (bx_n - x_n^2 y_n). 
\label{ode_approx}
\end{equation}
%
Rearranging the equation of $y_{n + 1}$ to solve for $y_n$ in \eqref{ode_approx} and substituting it into the one of $x_{n + 1}$ we obtain:
\begin{equation}
    \tau (1 - \tau) x_n^3 + \tau (\tau a - x_{n + 1} - y_{n + 1}) x_n^2 + (\tau b + \tau - 1) x_n + (x_{n + 1} - \tau a) = 0. \label{cubic}
\end{equation}
Equation \eqref{cubic} is a cubic for $x_n$ given $(x_{n + 1}, y_{n + 1})$ when $\tau \neq 1$.
By varying the parameters $a$, $b$ and $\tau$, 
we see the past states $(x_n, y_n)^T$ 
(also called ``inverses'' or ``preimages'')
may be multi-valued,
so that this discrete-time system is, in general, noninvertible. 
We fix $a = 1$ and consider how inverses will be changing (a) with $b$ for fixed $\tau = 0.15$; and (b) with  $\tau$, for fixed $b = 2$. 

In general, the neural network we are interested in is a mapping from 3D to 2D: $(x_{n + 1}, y_{n + 1})^T \approx \mathcal{N}(x_n, y_n; p)^T$, where 
$p \in \mathbb{R}$ is the parameter. 
The network dynamics will be parameter-dependent if we set $p \equiv b$, or timestep-dependent if $p \equiv \tau$. 
Considering the first layer of a MLP:
\begin{equation}
    W^{(0)} \begin{bmatrix}
    x_n \\
    y_n \\
    p \\
    \end{bmatrix} + b^{(0)} = (W^{(0)} (e_1 + e_2)) \begin{bmatrix}
    x_n \\
    y_n \\
    \end{bmatrix} + (p W^{(0)} e_3 + b^{(0)}), \label{eq: 3d_2d}
\end{equation}
where $e_{1, 2, 3} \in \mathbb{R}^3$ are indicator vectors. 
For fixed $p$ our network $\mathcal{N}$ can be thought of as an MLP mapping from $\mathbb{R}^2$ to $\mathbb{R}^2$, by slightly modifying the weights and biases in the first linear layer. Here, we trained two separate MLPs, with $b$ and $\tau$ dependence respectively.

\begin{figure}[H]
    \centering
    \includegraphics[width=0.6\textwidth]{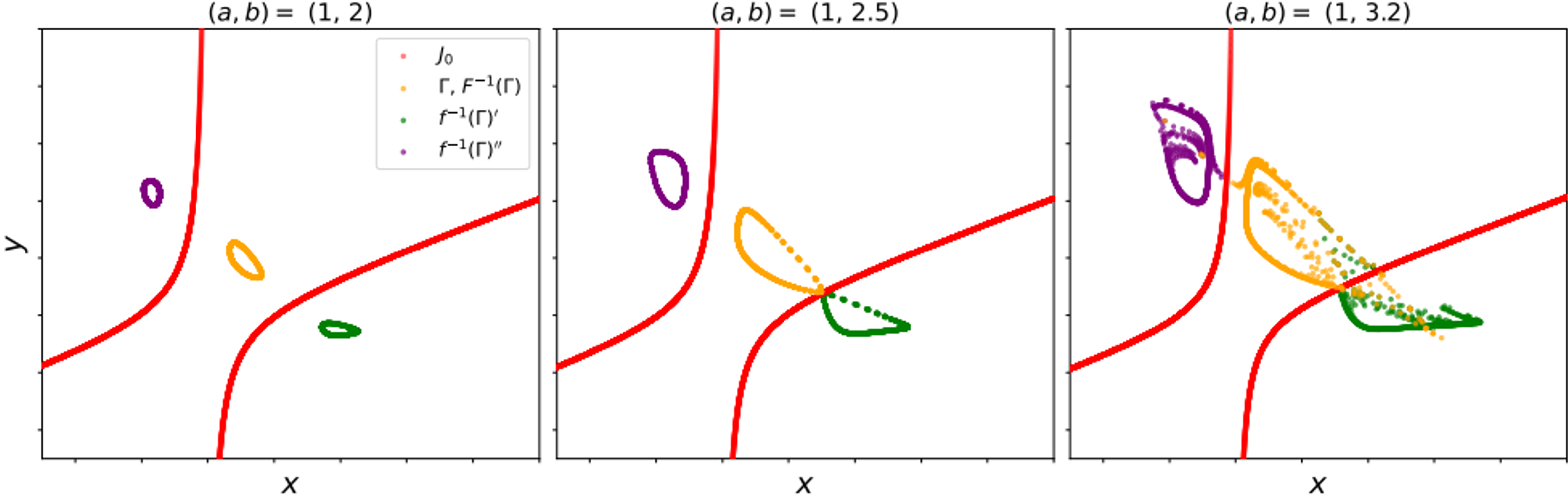}
    \caption{\small{Attractors and their multiple inverses for several parameter values of the Brusselator model. Notice the relation of the $J_0$ curves and the ``extra" preimages. When the attractor starts interacting with the $J_0$ and these extra preimages, the dynamic behavior degenerates quantitatively and qualitatively.}}
    \label{fig:inv}
\end{figure}

\begin{wrapfigure}{r}{0.5\textwidth}
    \centering

    \includegraphics[width=\linewidth]{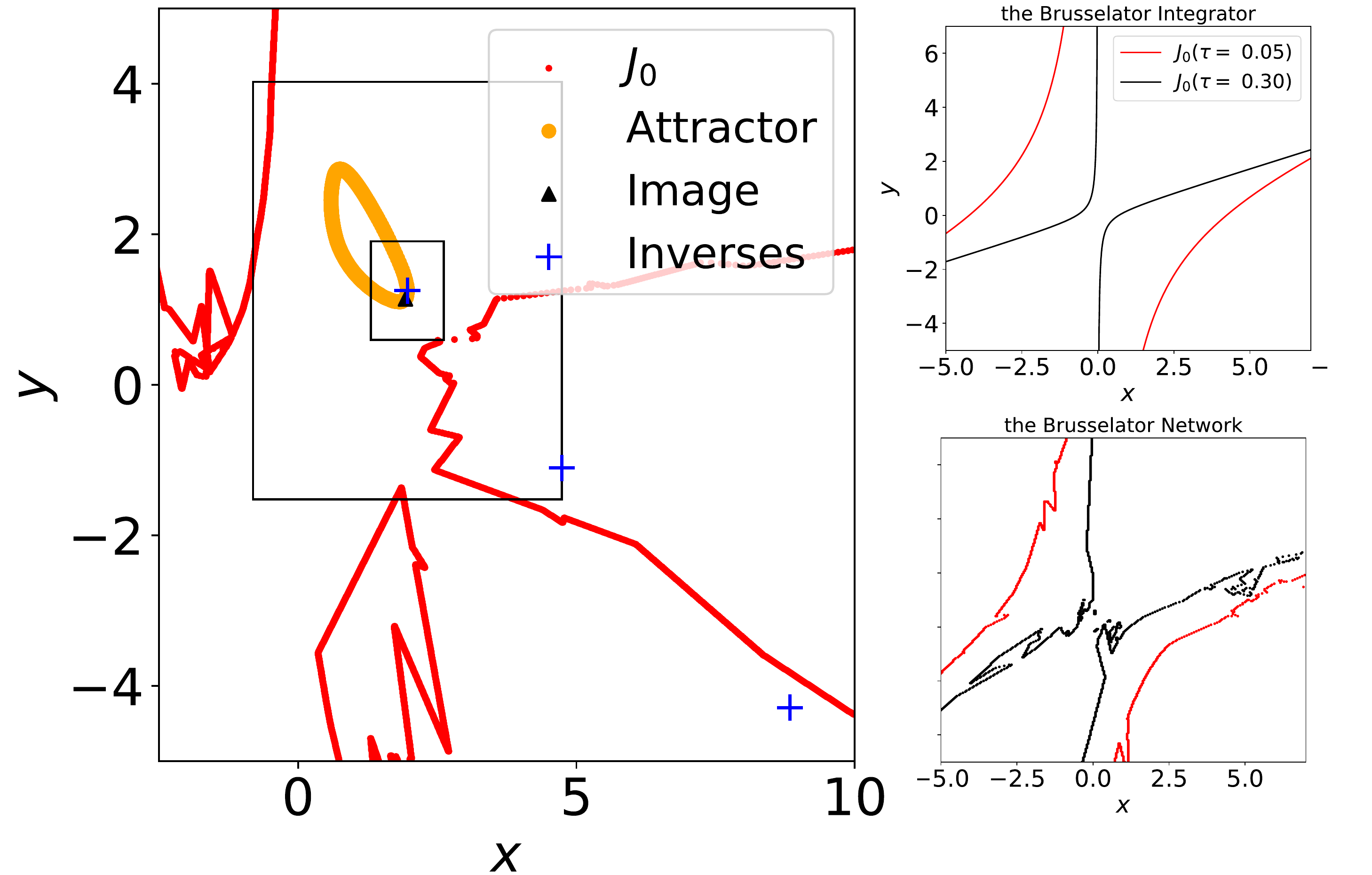}
    \caption{\small{\textbf{Left:} illustration of our solution to Problems \ref{problem 1} and \ref{problem 2} for the Brusselator network. For a random reference point on the attractor, we show the neighborhoods found by our algorithms. They clearly find the closest point on the $J_0$ curve / the closest ``extra preimage" of the point of interest.  \textbf{Right}: plots of $J_0$ curves at different $\tau$, for both the Euler integrator \textbf{(Top)} and our Brusselator ReLU network \textbf{(Bottom)}. Small timesteps lead to progressively remote $J_0$ curves. 
Notice also the piecewise linear nature of the $J_0$ curve for the ReLU network; its accurate computation is an interesting problem by itself.}}
    \label{2_inv}
\end{wrapfigure}

\subparagraph{Parameter-Dependent Inverses}
We start with a brief discussion of the dynamics and noninvertibility in the ground-truth system (see Figure \ref{fig:inv}). 
Consider an initial state located on the Brusselator attracting invariant circle (IC, in orange); we know this has at least one preimage {\em also on this IC}. 
In Figure \ref{fig:inv} we see that every point on the IC has three preimages: one still on the IC, and two additional inverses (in green and purple);
after one iteration, all three loops map to the orange one, and then remain forward invariant.
The phase space {\em folds} along the two branches of the $J_0$ curve (shown in red).
For lower values of $b$ (left), these three closed loops  do not intersect each other. As $b$ increases the (orange) attractor will become tangent to (center), and subsequently intersect $J_0$ (right), leading to mixing of the preimages.
At this point 
the predicted dynamics become nonphysical (beyond just inaccurate). 

After convergence of training, we employ our algorithm to obtain noninvertibility certificates for the resulting MLP, and plot results of $b = 2.1$ in the left subfigure of Figure \ref{2_inv}. 
In Figure \ref{2_inv}, we arbitrarily select one representative point, marked by a triangle ($\triangle$), on the attractor (the orange invariant circle); a nearby inverse {\em also} on the attractor, the {\em primal} inverse, is 
marked by a cross ($+$). Our algorithm will produce two regions for this point, one for each of our problems (squares of constant $L_{\infty}$ distance in 2D). As a sanity check, we also compute the $J_0$ sets (the red point), as well as a few additional inverses, beyond the primal ones
with the help of numerical root solver and automatic differentiation (\cite{autodiff}). Clearly, the smaller square neighborhood ``just hits" the $J_0$ curve, while the larger one extends to the closest nonprimal inverse of the attractor. 

\subparagraph{Timestep-Dependent Inverses} 
In the right two subfigures of Figure \ref{2_inv},
we explore the effect of varying the time horizon $\tau$.
We compare a single Euler step of the ground truth ODE
to the MLP approximating the same flowmap,
and find that, in both, smaller time horizons
lead to larger regions of invertibility.

\begin{figure}[H]
    \centering
    \includegraphics[width=0.37\linewidth]{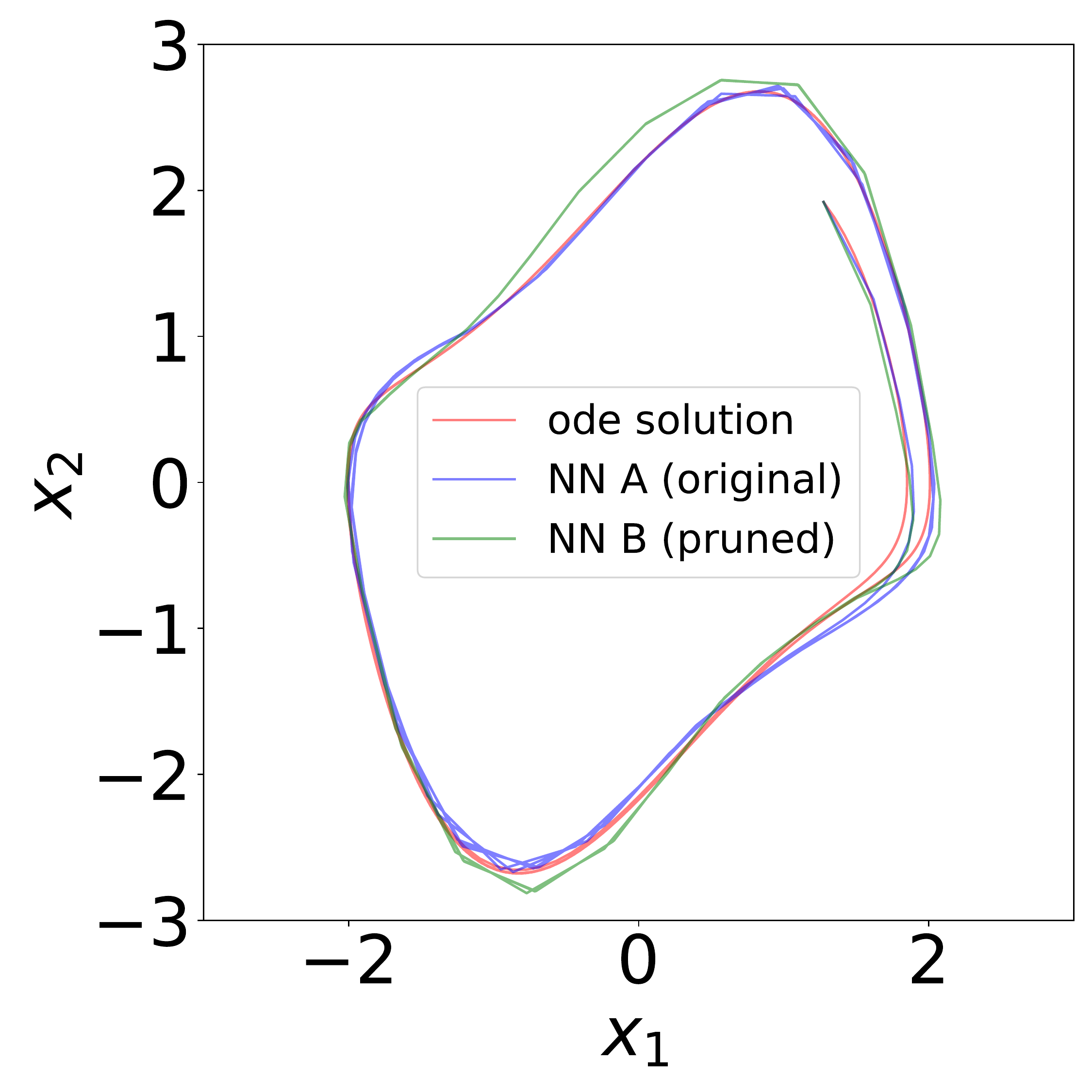}
    \includegraphics[width=0.5\linewidth]{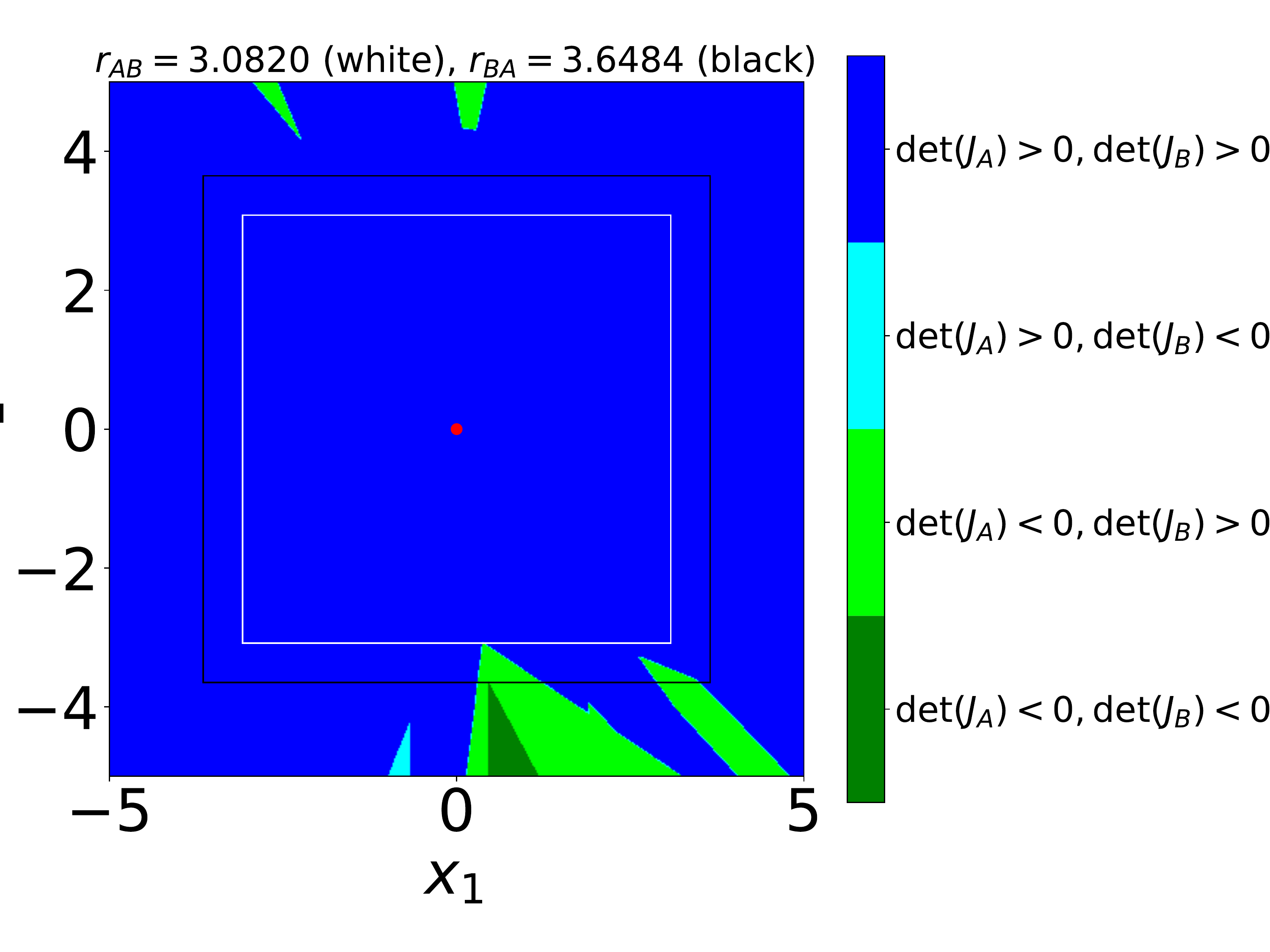}
    \caption{\small{Left: Trajectories of the ODE solution for the Van der Pol system (red), and their discrete-time neural network approximations (blue and green). All three trajectories begin at the same initial state; the ODE solution is smooth (continuous-time), the other two use straight lines between consecutive states (discrete-time). However, it is clear all three systems have nearby attractors, indicating good performance of the network and its pruned version. Right: visualization of MILP computation results, along with the sign of the Jacobian values of the networks on the grid points of the input domain. Here, the center of the region is marked red, while the white and black boundaries quantify the region of mappability between outputs of network A and network B.}}
    \label{prune}
\end{figure}

\paragraph{Network Transformation Example: Learning the Van der Pol Equation}
Here, to test our algorithm on network transformation problem \ref{problem 3}, we trained two networks on the same regression task. Our data comes from the 2D Van der Pol equation $dx_1 / dt = x_2, dx_2 / dt = \mu (1 - x_1^2) x_2 - x_1$, where the input and output are the initial and final states of 1000 solution trajectories with time duration 0.2 for $\mu = 1$, when a stable limit cycle exists.
The initial states are uniformly sampled in the region $[-3, 3] \times [-3, 3]$. The neural network A used to learn the time series is a 2-32-32-2 MLP, while the neural network B is a sparse version of A, where half of the weight entries are pruned (set to zero) 
based on \cite{prune}. To visualize the performances of the networks, two trajectories generated by respectively iterating the network functions for fixed times from a given initial state have been plotted in the left subplot of Figure \ref{prune}. The ODE solution trajectory starting at the same initial state with same time duration is also shown.
We see that both network functions A and B exhibit long-term oscillations,
though the shapes of the attractors have small visual differences from the true ODE solution (the red curve).

These two network functions were then used to test the correctness of the algorithm for the problem \ref{problem 3}. Here we chose the center points $x_c = (0, 0)^T$, computed and plotted the mappable regions for two subcases (see right subfigure of Figure \ref{prune}): the output of network $B$ is a function of  the output of network $A$ (the square with white bounds centered at the red point, radius 3.0820), and vice versa (the square with black bounds centered at the red point, radius 3.6484).
For validation
we also computed the Jacobian values of network $A$ and network $B$ on every grid point of the input domain, and shown that the white square touches the $J_0$ curve of network $A$, while the black square touches the $J_0$ curve of network $B$. 
%
Inside the black square the Jacobian of network $B$ remains positive, so that network $B$ is invertible (i.e. the existence of the mapping from $f_B(x)$ to $x$, or equivalently, $f_B^{-1} (x)$); therefore we can find the mapping from $f_B(x)$ to $f_A(x)$ by composing the mapping from $f_B(x)$ to $x$ and the mapping from $x$ to $f_A(x)$ (the function $f_A(x)$ itself). The size of the white square can be similarly rationalized, validating our computation. 

\begin{table}[H]
\centering
\begin{tabular}{c|ccc|ccc|ccc}
\toprule
Sparsity & \multicolumn{3}{c|}{40 \%} & \multicolumn{3}{|c|}{50 \%} & \multicolumn{3}{|c}{60 \%} \\
\midrule
Network $B$ & $B_1$ & $B_2$ & $B_3$ & $B_4$ & $B_5$ & $B_6$ & $B_7$ & $B_8$ & $B_9$ \\
\midrule
$r_{AB}$ & 3.0820 & 3.0820 & 3.0820 & 3.0820 & 3.0820 & 3.0820 & 3.0820 & 3.0820 & 3.0820 \\
\midrule
$r_{BA}$ & 3.4609 & 3.1055 & 3.8555 & 3.6484 & 2.6523 & 3.8203 & 3.6328 & 3.9727 & 4.5547 \\
\bottomrule
\end{tabular}
\caption{\small{The radii of the mappable regions between the original network $A$ and its pruned versions $B$.}}
\label{tab:pruneResults}
\end{table}

As a sanity check, we consructed eight more pruned networks; two of them have $50 \%$ sparsity (networks $B_5$ and $B_6$), three have $40 \%$ sparsity (networks $B_1, B_2$ and $B_3$) and the others have $60 \%$ sparsity (networks $B_7, B_8$ and $B_9$). Above, we discussed network
$B_4$. 
For each pruned network, we computed the radii of the regions of interest (aka $r_{AB}$ and $r_{BA}$). The results are listed in Table \ref{tab:pruneResults}. All  pruned networks $\{B_i\}$ share the same radii $r_{AB}$, consistent with the invertibility of $A$ itself. Since $r_A = 3.0820$, $A$ is invertible in the ball we computed, and 
the existence of the mapping $y_A \mapsto y_B$ by composition of $y_A \mapsto x$ and $x \mapsto y_B$. 
In our work the input and output dimensions are the same (e.g. $m = n$ in Problem \ref{problem 3}); this condition is not restrictive, and our algorithm can be possibly extended to classification problems, where in general  $m \gg n$.

\section{Conclusions}
In this paper, we addressed the issue of noninvertibility that arises in discrete-time dynamical systems and neural networks performing time-series related tasks. We highlighted the potential pathological consequences of such noninvertibility, which extend beyond prediction inaccuracies and affect the predicted dynamics of the networks. Moreover, we extended our analysis to transformations between different neural networks and formulated three problems that provide a quantifiable assessment of local invertibility for any arbitrarily selected input. For functions such as MLPs with ReLU activations, we formulated these problems as mixed-integer programs and performed experiments on regression tasks; we also extended our algorithm to Resnets.

In future work, we aim to develop structure-exploiting methods that can globally solve these mixed-integer programs more efficiently for larger networks. Additionally, given the linearity of convolution and average pooling operations and the piecewise linearity of max pooling, we plan to adapt our algorithm to convolutional neural networks like AlexNet  (\cite{alex}) and VGG (\cite{VGG}). Our successful application of the algorithm to ResNet architectures (\cite{resnet}) holds promise for applicability to recursive architectures (\cite{lu18d, Weinan2017APO}) such as fractal networks (\cite{larsson2017fractalnet}), poly-inception networks (\cite{zhang2016polynet}), and RevNet (\cite{gomez2017reversible}). Furthermore, we are working on making the algorithm practical for continuous differentiable activations such as tanh or Swish (\cite{swish}), and other piecewise activations such as Gaussian Error Linear Units (GELUs, \cite{gelu}). Finally, we are particularly interested in exploring the case where the input and output domains have different dimensions, such as in classifiers.

\blfootnote{Full text is available at: \url{https://arxiv.org/abs/2301.11783}.}
\bibliography{bi.bib}

\begin{thebibliography}{46}
\providecommand{\natexlab}[1]{#1}
\providecommand{\url}[1]{\texttt{#1}}
\expandafter\ifx\csname urlstyle\endcsname\relax
  \providecommand{\doi}[1]{doi: #1}\else
  \providecommand{\doi}{doi: \begingroup \urlstyle{rm}\Url}\fi

\bibitem[Adomaitis and Kevrekidis(1991)]{Adomaitis1991NoninvertibilityAT}
R.~Adomaitis and I.~Kevrekidis.
\newblock Noninvertibility and the structure of basins of attraction in a model
  adaptive control system.
\newblock \emph{Journal of Nonlinear Science}, 1:\penalty0 95--105, 1991.

\bibitem[Ardizzone et~al.(2018)Ardizzone, Kruse, Wirkert, Rahner, Pellegrini,
  Klessen, Maier-Hein, Rother, and K{\"o}the]{ardizzone2018analyzing}
Lynton Ardizzone, Jakob Kruse, Sebastian Wirkert, Daniel Rahner, Eric~W
  Pellegrini, Ralf~S Klessen, Lena Maier-Hein, Carsten Rother, and Ullrich
  K{\"o}the.
\newblock Analyzing inverse problems with invertible neural networks.
\newblock \emph{arXiv preprint arXiv:1808.04730}, 2018.

\bibitem[Ardizzone et~al.(2019)Ardizzone, Kruse, Wirkert, Rahner, Pellegrini,
  Klessen, Maier-Hein, Rother, and K{\"o}the]{Ardizzone2019AnalyzingIP}
Lynton Ardizzone, Jakob Kruse, Sebastian~J. Wirkert, D.~Rahner, Eric~W.
  Pellegrini, R.~Klessen, L.~Maier-Hein, C.~Rother, and U.~K{\"o}the.
\newblock Analyzing inverse problems with invertible neural networks.
\newblock \emph{ArXiv}, abs/1808.04730, 2019.

\bibitem[Baydin et~al.(2017)Baydin, Pearlmutter, Radul, and Siskind]{autodiff}
At\i{}l\i{}m~G\"{u}nes Baydin, Barak~A. Pearlmutter, Alexey~Andreyevich Radul,
  and Jeffrey~Mark Siskind.
\newblock Automatic differentiation in machine learning: A survey.
\newblock \emph{J. Mach. Learn. Res.}, 18\penalty0 (1):\penalty0 5595–5637,
  January 2017.
\newblock ISSN 1532-4435.

\bibitem[Beasley(1996)]{10.5555/247975}
J.~E. Beasley, editor.
\newblock \emph{Advances in Linear and Integer Programming}.
\newblock Oxford University Press, Inc., USA, 1996.
\newblock ISBN 0198538561.

\bibitem[Behrmann et~al.(2018)Behrmann, Dittmer, Fernsel, and
  Maass]{Behrmann2018AnalysisOI}
Jens Behrmann, S{\"o}ren Dittmer, Pascal Fernsel, and P.~Maass.
\newblock Analysis of invariance and robustness via invertibility of
  relu-networks.
\newblock \emph{ArXiv}, abs/1806.09730, 2018.

\bibitem[Behrmann et~al.(2019)Behrmann, Grathwohl, Chen, Duvenaud, and
  Jacobsen]{behrmann2019invertible}
Jens Behrmann, Will Grathwohl, Ricky T.~Q. Chen, David Duvenaud, and
  Joern-Henrik Jacobsen.
\newblock Invertible residual networks.
\newblock In Kamalika Chaudhuri and Ruslan Salakhutdinov, editors,
  \emph{Proceedings of the 36th International Conference on Machine Learning},
  volume~97 of \emph{Proceedings of Machine Learning Research}, pages 573--582.
  PMLR, 09--15 Jun 2019.
\newblock URL \url{http://proceedings.mlr.press/v97/behrmann19a.html}.

\bibitem[Behrmann et~al.(2021)Behrmann, Vicol, Wang, Grosse, and
  Jacobsen]{behrmann2020understanding}
Jens Behrmann, Paul Vicol, Kuan-Chieh Wang, Roger Grosse, and Joern-Henrik
  Jacobsen.
\newblock Understanding and mitigating exploding inverses in invertible neural
  networks.
\newblock In Arindam Banerjee and Kenji Fukumizu, editors, \emph{Proceedings of
  The 24th International Conference on Artificial Intelligence and Statistics},
  volume 130 of \emph{Proceedings of Machine Learning Research}, pages
  1792--1800. PMLR, 13--15 Apr 2021.
\newblock URL \url{http://proceedings.mlr.press/v130/behrmann21a.html}.

\bibitem[Chang et~al.(2018)Chang, Meng, Haber, Ruthotto, Begert, and
  Holtham]{chang2017reversible}
Bo~Chang, Lili Meng, Eldad Haber, Lars Ruthotto, David Begert, and Elliot
  Holtham.
\newblock Reversible architectures for arbitrarily deep residual neural
  networks.
\newblock In Sheila~A. McIlraith and Kilian~Q. Weinberger, editors,
  \emph{Proceedings of the Thirty-Second {AAAI} Conference on Artificial
  Intelligence, (AAAI-18), the 30th innovative Applications of Artificial
  Intelligence (IAAI-18), and the 8th {AAAI} Symposium on Educational Advances
  in Artificial Intelligence (EAAI-18), New Orleans, Louisiana, USA, February
  2-7, 2018}, pages 2811--2818. {AAAI} Press, 2018.
\newblock URL
  \url{https://www.aaai.org/ocs/index.php/AAAI/AAAI18/paper/view/16517}.

\bibitem[Chen et~al.(2018)Chen, Rubanova, Bettencourt, and
  Duvenaud]{chen2018neural}
Ricky T.~Q. Chen, Yulia Rubanova, Jesse Bettencourt, and David Duvenaud.
\newblock Neural ordinary differential equations.
\newblock In \emph{Proceedings of the 32nd International Conference on Neural
  Information Processing Systems}, NIPS'18, page 6572–6583, Red Hook, NY,
  USA, 2018. Curran Associates Inc.

\bibitem[Chen et~al.(2019)Chen, Behrmann, Duvenaud, and
  Jacobsen]{chen2019residualflows}
Ricky~TQ Chen, Jens Behrmann, David Duvenaud, and J{\"o}rn-Henrik Jacobsen.
\newblock Residual flows for invertible generative modeling.
\newblock In \emph{Neural Information Processing Systems}, 2019.
\newblock URL \url{https://arxiv.org/abs/1906.02735}.

\bibitem[Dinh et~al.(2015)Dinh, Krueger, and Bengio]{dinh2015nice}
Laurent Dinh, David Krueger, and Yoshua Bengio.
\newblock {NICE:} non-linear independent components estimation.
\newblock In Yoshua Bengio and Yann LeCun, editors, \emph{3rd International
  Conference on Learning Representations, {ICLR} 2015, San Diego, CA, USA, May
  7-9, 2015, Workshop Track Proceedings}, 2015.
\newblock URL \url{http://arxiv.org/abs/1410.8516}.

\bibitem[Dinh et~al.(2017)Dinh, Sohl{-}Dickstein, and Bengio]{dinh2017density}
Laurent Dinh, Jascha Sohl{-}Dickstein, and Samy Bengio.
\newblock Density estimation using real {NVP}.
\newblock In \emph{5th International Conference on Learning Representations,
  {ICLR} 2017, Toulon, France, April 24-26, 2017, Conference Track
  Proceedings}. OpenReview.net, 2017.
\newblock URL \url{https://openreview.net/forum?id=HkpbnH9lx}.

\bibitem[Donahue and Simonyan(2019)]{donahue2019large}
Jeff Donahue and Karen Simonyan.
\newblock Large scale adversarial representation learning.
\newblock \emph{Advances in neural information processing systems}, 32, 2019.

\bibitem[E(2017)]{Weinan2017APO}
Weinan E.
\newblock A proposal on machine learning via dynamical systems.
\newblock \emph{Communications in Mathematics and Statistics}, 5\penalty0
  (1):\penalty0 1--11, 3 2017.
\newblock \doi{10.1007/s40304-017-0103-z}.
\newblock Dedicated to Professor Chi-Wang Shu on the occasion of his 60th
  birthday.

\bibitem[Frouzakis et~al.(1997)Frouzakis, Gardini, Kevrekidis, Millerioux, and
  Mira]{Frouzakis}
Christos~E. Frouzakis, Laura Gardini, Ioannis~G. Kevrekidis, Gilles Millerioux,
  and Christian Mira.
\newblock On some properties of invariant sets of two-dimensional noninvertible
  maps.
\newblock \emph{International Journal of Bifurcation and Chaos}, 07\penalty0
  (06):\penalty0 1167--1194, 1997.
\newblock \doi{10.1142/S0218127497000972}.
\newblock URL \url{https://doi.org/10.1142/S0218127497000972}.

\bibitem[Gicquel et~al.(1998)Gicquel, Anderson, and Kevrekidis]{GICQUEL19988}
N.~Gicquel, J.S. Anderson, and I.G. Kevrekidis.
\newblock Noninvertibility and resonance in discrete-time neural networks for
  time-series processing.
\newblock \emph{Physics Letters A}, 238\penalty0 (1):\penalty0 8--18, 1998.
\newblock ISSN 0375-9601.
\newblock \doi{https://doi.org/10.1016/S0375-9601(97)00753-6}.
\newblock URL
  \url{https://www.sciencedirect.com/science/article/pii/S0375960197007536}.

\bibitem[Gomez et~al.(2017)Gomez, Ren, Urtasun, and
  Grosse]{gomez2017reversible}
Aidan~N Gomez, Mengye Ren, Raquel Urtasun, and Roger~B Grosse.
\newblock The reversible residual network: Backpropagation without storing
  activations.
\newblock \emph{Advances in neural information processing systems}, 30, 2017.

\bibitem[{Gurobi Optimization, LLC}(2023)]{gurobi}
{Gurobi Optimization, LLC}.
\newblock {Gurobi Optimizer Reference Manual}, 2023.
\newblock URL \url{https://www.gurobi.com}.

\bibitem[He et~al.(2016)He, Zhang, Ren, and Sun]{resnet}
Kaiming He, Xiangyu Zhang, Shaoqing Ren, and Jian Sun.
\newblock Deep residual learning for image recognition.
\newblock In \emph{2016 IEEE Conference on Computer Vision and Pattern
  Recognition (CVPR)}, pages 770--778, 2016.
\newblock \doi{10.1109/CVPR.2016.90}.

\bibitem[Hein and Andriushchenko(2017)]{hein2017formal}
Matthias Hein and Maksym Andriushchenko.
\newblock Formal guarantees on the robustness of a classifier against
  adversarial manipulation.
\newblock In \emph{Advances in Neural Information Processing Systems}, pages
  2266--2276, 2017.

\bibitem[Hendrycks and Gimpel(2016)]{gelu}
Dan Hendrycks and Kevin Gimpel.
\newblock Bridging nonlinearities and stochastic regularizers with gaussian
  error linear units.
\newblock \emph{CoRR}, abs/1606.08415, 2016.
\newblock URL \url{http://arxiv.org/abs/1606.08415}.

\bibitem[Jacobsen et~al.(2018)Jacobsen, Behrmann, Zemel, and
  Bethge]{jacobsen2018excessive}
J{\"o}rn-Henrik Jacobsen, Jens Behrmann, Richard Zemel, and Matthias Bethge.
\newblock Excessive invariance causes adversarial vulnerability.
\newblock \emph{arXiv preprint arXiv:1811.00401}, 2018.

\bibitem[Jaeger(2014)]{Jaeger2014ControllingRN}
H.~Jaeger.
\newblock Controlling recurrent neural networks by conceptors.
\newblock \emph{ArXiv}, abs/1403.3369, 2014.

\bibitem[Krizhevsky et~al.(2017)Krizhevsky, Sutskever, and Hinton]{alex}
Alex Krizhevsky, Ilya Sutskever, and Geoffrey~E. Hinton.
\newblock Imagenet classification with deep convolutional neural networks.
\newblock \emph{Commun. ACM}, 60\penalty0 (6):\penalty0 84–90, May 2017.
\newblock ISSN 0001-0782.
\newblock \doi{10.1145/3065386}.
\newblock URL \url{https://doi.org/10.1145/3065386}.

\bibitem[Land and Doig(1960)]{bandb}
A.~H. Land and A.~G. Doig.
\newblock An automatic method of solving discrete programming problems.
\newblock \emph{Econometrica}, 28\penalty0 (3):\penalty0 497--520, 1960.
\newblock ISSN 00129682, 14680262.
\newblock URL \url{http://www.jstor.org/stable/1910129}.

\bibitem[Larsson et~al.(2017)Larsson, Maire, and
  Shakhnarovich]{larsson2017fractalnet}
Gustav Larsson, Michael Maire, and Gregory Shakhnarovich.
\newblock Fractalnet: Ultra-deep neural networks without residuals.
\newblock In \emph{ICLR}, 2017.

\bibitem[Lu et~al.(2018)Lu, Zhong, Li, and Dong]{lu18d}
Yiping Lu, Aoxiao Zhong, Quanzheng Li, and Bin Dong.
\newblock Beyond finite layer neural networks: Bridging deep architectures and
  numerical differential equations.
\newblock In Jennifer Dy and Andreas Krause, editors, \emph{Proceedings of the
  35th International Conference on Machine Learning}, volume~80 of
  \emph{Proceedings of Machine Learning Research}, pages 3282--3291, Stockholm,
  Stockholm Sweden, 10--15 Jul 2018. PMLR.
\newblock URL \url{http://proceedings.mlr.press/v80/lu18d.html}.

\bibitem[MacKay et~al.(2018)MacKay, Vicol, Ba, and
  Grosse]{10.5555/3327546.3327578}
Matthew MacKay, Paul Vicol, Jimmy Ba, and Roger Grosse.
\newblock Reversible recurrent neural networks.
\newblock In \emph{Proceedings of the 32nd International Conference on Neural
  Information Processing Systems}, NIPS'18, page 9043–9054, Red Hook, NY,
  USA, 2018. Curran Associates Inc.

\bibitem[Radev et~al.(2020)Radev, Mertens, Voss, Ardizzone, and
  Köthe]{9298920}
Stefan~T. Radev, Ulf~K. Mertens, Andreas Voss, Lynton Ardizzone, and Ullrich
  Köthe.
\newblock Bayesflow: Learning complex stochastic models with invertible neural
  networks.
\newblock \emph{IEEE Transactions on Neural Networks and Learning Systems},
  pages 1--15, 2020.
\newblock \doi{10.1109/TNNLS.2020.3042395}.

\bibitem[Ramachandran et~al.(2017)Ramachandran, Zoph, and Le]{swish}
Prajit Ramachandran, Barret Zoph, and Quoc~V. Le.
\newblock Swish: a self-gated activation function.
\newblock \emph{CoRR}, 2017.
\newblock URL \url{http://arxiv.org/abs/1710.05941v1}.

\bibitem[Rico-Martinez et~al.(1993)Rico-Martinez, Kevrekidis, and
  Adomaitis]{298587}
R.~Rico-Martinez, I.G. Kevrekidis, and R.A. Adomaitis.
\newblock Noninvertibility in neural networks.
\newblock In \emph{IEEE International Conference on Neural Networks}, pages
  382--386 vol.1, 1993.
\newblock \doi{10.1109/ICNN.1993.298587}.

\bibitem[Simonyan and Zisserman(2015)]{VGG}
Karen Simonyan and Andrew Zisserman.
\newblock Very deep convolutional networks for large-scale image recognition.
\newblock In Yoshua Bengio and Yann LeCun, editors, \emph{3rd International
  Conference on Learning Representations, {ICLR} 2015, San Diego, CA, USA, May
  7-9, 2015, Conference Track Proceedings}, 2015.
\newblock URL \url{http://arxiv.org/abs/1409.1556}.

\bibitem[Song et~al.(2019)Song, Meng, and Ermon]{song2019mintnet}
Yang Song, Chenlin Meng, and Stefano Ermon.
\newblock Mintnet: Building invertible neural networks with masked
  convolutions.
\newblock In H.~Wallach, H.~Larochelle, A.~Beygelzimer, F.~d\textquotesingle
  Alch\'{e}-Buc, E.~Fox, and R.~Garnett, editors, \emph{Advances in Neural
  Information Processing Systems}, volume~32. Curran Associates, Inc., 2019.
\newblock URL
  \url{https://proceedings.neurips.cc/paper/2019/file/70a32110fff0f26d301e58ebbca9cb9f-Paper.pdf}.

\bibitem[Szegedy et~al.(2013)Szegedy, Zaremba, Sutskever, Bruna, Erhan,
  Goodfellow, and Fergus]{szegedy2013intriguing}
Christian Szegedy, Wojciech Zaremba, Ilya Sutskever, Joan Bruna, Dumitru Erhan,
  Ian Goodfellow, and Rob Fergus.
\newblock Intriguing properties of neural networks.
\newblock \emph{arXiv preprint arXiv:1312.6199}, 2013.

\bibitem[Takens(1981)]{Takens1981}
Floris Takens.
\newblock Detecting strange attractors in turbulence.
\newblock In \emph{Dynamical Systems and Turbulence, Warwick 1980: Proceedings
  of a Symposium Held at the University of Warwick 1979/80}, pages 366--381,
  Berlin, Heidelberg, 1981. Springer Berlin Heidelberg.
\newblock ISBN 978-3-540-38945-3.
\newblock \doi{10.1007/BFb0091924}.
\newblock URL \url{https://doi.org/10.1007/BFb0091924}.

\bibitem[Teshima et~al.(2020)Teshima, Ishikawa, Tojo, Oono, Ikeda, and
  Sugiyama]{teshima2020couplingbased}
Takeshi Teshima, Isao Ishikawa, Koichi Tojo, Kenta Oono, Masahiro Ikeda, and
  Masashi Sugiyama.
\newblock Coupling-based invertible neural networks are universal
  diffeomorphism approximators, 2020.

\bibitem[Tjeng et~al.(2017)Tjeng, Xiao, and Tedrake]{tjeng2017evaluating}
Vincent Tjeng, Kai Xiao, and Russ Tedrake.
\newblock Evaluating robustness of neural networks with mixed integer
  programming.
\newblock \emph{arXiv preprint arXiv:1711.07356}, 2017.

\bibitem[Tram{\`e}r et~al.(2020)Tram{\`e}r, Behrmann, Carlini, Papernot, and
  Jacobsen]{tramer2020fundamental}
Florian Tram{\`e}r, Jens Behrmann, Nicholas Carlini, Nicolas Papernot, and
  J{\"o}rn-Henrik Jacobsen.
\newblock Fundamental tradeoffs between invariance and sensitivity to
  adversarial perturbations.
\newblock In \emph{International Conference on Machine Learning}, pages
  9561--9571. PMLR, 2020.

\bibitem[Tyson(1973)]{doi:10.1063/1.1679748}
John~J. Tyson.
\newblock Some further studies of nonlinear oscillations in chemical systems.
\newblock \emph{The Journal of Chemical Physics}, 58\penalty0 (9):\penalty0
  3919--3930, 1973.
\newblock \doi{10.1063/1.1679748}.
\newblock URL \url{https://doi.org/10.1063/1.1679748}.

\bibitem[Wang et~al.(2018)Wang, Pei, Whitehouse, Yang, and
  Jana]{wang2018efficient}
Shiqi Wang, Kexin Pei, Justin Whitehouse, Junfeng Yang, and Suman Jana.
\newblock Efficient formal safety analysis of neural networks.
\newblock In \emph{Advances in Neural Information Processing Systems}, pages
  6367--6377, 2018.

\bibitem[Weng et~al.(2018)Weng, Zhang, Chen, Song, Hsieh, Boning, Dhillon, and
  Daniel]{weng2018towards}
Tsui-Wei Weng, Huan Zhang, Hongge Chen, Zhao Song, Cho-Jui Hsieh, Duane Boning,
  Inderjit~S Dhillon, and Luca Daniel.
\newblock Towards fast computation of certified robustness for relu networks.
\newblock \emph{arXiv preprint arXiv:1804.09699}, 2018.

\bibitem[Wong and Kolter(2018)]{wong2018provable}
Eric Wong and Zico Kolter.
\newblock Provable defenses against adversarial examples via the convex outer
  adversarial polytope.
\newblock In \emph{International Conference on Machine Learning}, pages
  5286--5295, 2018.

\bibitem[Zhang et~al.(2018)Zhang, Weng, Chen, Hsieh, and
  Daniel]{zhang2018efficient}
Huan Zhang, Tsui-Wei Weng, Pin-Yu Chen, Cho-Jui Hsieh, and Luca Daniel.
\newblock Efficient neural network robustness certification with general
  activation functions.
\newblock In \emph{Advances in Neural Information Processing Systems}, pages
  4939--4948, 2018.

\bibitem[Zhang et~al.(2016)Zhang, Li, Loy, and Lin]{zhang2016polynet}
Xingcheng Zhang, Zhizhong Li, Chen~Change Loy, and Dahua Lin.
\newblock Polynet: A pursuit of structural diversity in very deep networks.
\newblock \emph{arXiv preprint arXiv:1611.05725}, 2016.

\bibitem[Zhu and Gupta(2018)]{prune}
Michael Zhu and Suyog Gupta.
\newblock To prune, or not to prune: Exploring the efficacy of pruning for
  model compression.
\newblock In \emph{6th International Conference on Learning Representations,
  {ICLR} 2018, Vancouver, BC, Canada, April 30 - May 3, 2018, Workshop Track
  Proceedings}. OpenReview.net, 2018.
\newblock URL \url{https://openreview.net/forum?id=Sy1iIDkPM}.

\end{thebibliography}

\newpage
\appendix

\section{Further Discussions}

\renewcommand\thefigure{A.\arabic{figure}}
\renewcommand\thetable{A.\arabic{table}}

\setcounter{figure}{0}
\setcounter{table}{0}

\subsection{Invertibilty in View of Lipschitz Constants}
One can consider the neural network inversion problem in terms of Lipschitz continuity and the Lipschitz constant. Indeed, quantifying invertibility of a neural network (more generally, a function) is intimately connected with its Lipschitz constant.

\begin{definition}[Lipschitz continuity and Lipschitz constant]
A function $F: \mathcal{B} \subseteq \mathbb{R}^m \mapsto \mathbb{R}^m$ is Lipschitz continuous on $\mathcal{B}$ if there exists a non-negative constant
$L \geq 0$ such that
\begin{equation}
    \frac{||F(x_1) - F(x_2)||}{||x_1 - x_2||} \leq L,  \quad \forall x_1, x_2 \in \mathcal{B}, x_1 \neq x_2. \label{eq: Lip_const}
\end{equation}
The smallest such $L$ is called the Lipschitz constant of $F$, $L = \operatorname{Lip}(F)$.
\label{def: lip}
\end{definition}

A generalization for Definition \ref{def: lip} is the bi-Lipschitz map defined as follows.

\begin{definition}[bi-Lipschitz continuity and bi-Lipschitz constant]
Suppose $F: \mathcal{B} \subseteq \mathbb{R}^m \mapsto \mathbb{R}^m$ is globally Lipschitz continuous with Lipschitz constant $L$. Now we define another nonnegative constant $L' \geq 0$ such that
\begin{equation}
     L' \leq \frac{||F(x_1) - F(x_2)||}{||x_1 - x_2||}, \quad \forall x_1, x_2 \in \mathbb{R}^m, x_1 \neq x_2. \label{eq: bi_Lip}
\end{equation}
\label{def: bi-lip}
\end{definition}
\noindent
If the largest such $L'$ is strictly positive, then \eqref{eq: bi_Lip} shows $F$ is invertible on $\mathcal{B}$ due to $F(x_1) \neq F(x_2)$ given $x_1 \neq x_2$. Moreover, one could easily derive $(1 / L') = \operatorname{Lip}(F^{-1})$, where $F^{-1}$ is the inverse function of $F$. We also say $F$ is bi-Lipschitz continuous in this case, with bi-Lipschitz constant $L^{\ast} = \max \left\{L, 1 / L' \right\}$.

\subsection{Structure of Preimages for the Learned Map of the Brusselator Flow}

As discussed in the main paper, we trained a network to approximate the time-$\tau$ Euler map \eqref{ode_approx_appendix} for the Brusselator.
%
%
The attractor 
(locus of long-term image points)
is a small amplitude, stable invariant circle (IC), the discrete time analog of the ODE stable limit cycle. We mark four representative points on it ($Q$, $R$, $S$, and $T$) and divide it into parts $A$, $B_1$, $B_2$, and $C$ between these points, so as to facilitate the description of the dynamics and its multiple (due to noninvertibility) preimages. 
The locus of red points (the locus on which the determinant of the Jacobian of the network changes sign, or, in the language of noninvertible systems, the $J_0$ curve) separates state space here into five distinct regions \regionOne{},\,\ldots,\,\regionFive{}, each with different preimage behavior,
as illustrated in Figure \ref{fig:smoothPreimageStructure}.
For smooth maps, like the Brusselator forward Euler discretization or a $\tanh$ activation neural network, $J_0$ is the locus of points for which the determinant of the map Jacobian is zero (and therefore, the map is singular).
In those cases, the curve is easy to compute through continuation algorithms. 
For ReLU activations, however,  this locus is nontrivial to compute through algebraic solvers, and piecewise smooth computational techniques or brute force exploration must be used to locate it;
see the inset in Figure, where the color intensity indicates the 
magnitude, red for positive and blue for negative, of the map Jacobian determinant. 
After we locate the $J_0$ points however, we see that they define the \regionOne{} through \regionFive{} (and implicitly, through forward iteration of the $J_0$ curve, regions $A$ through $C$ on the IC):

\begin{figure}[H]
    \centering
    \includegraphics[width=0.7\linewidth]{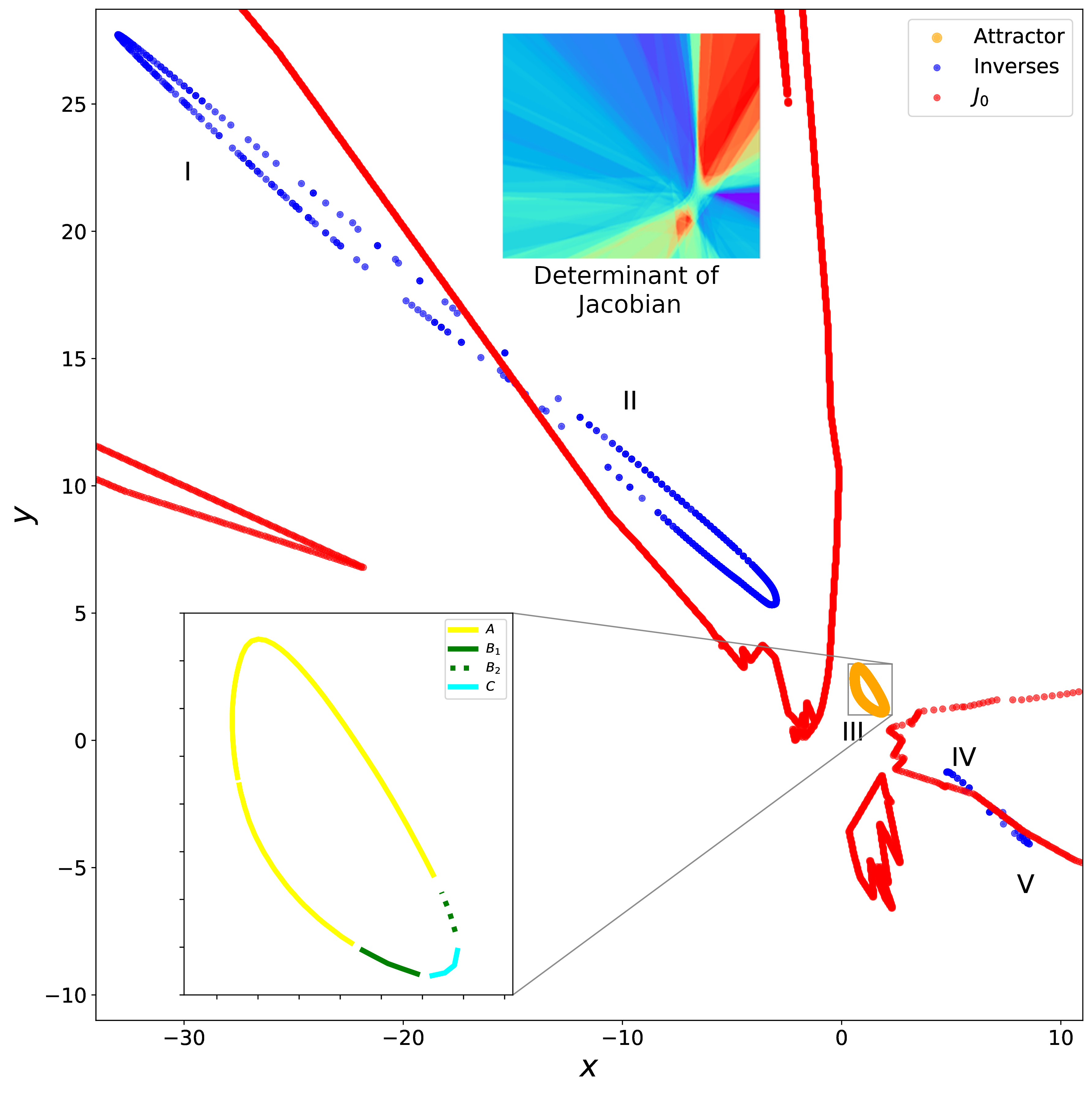}
    \hspace*{0.3cm}\includegraphics[width=0.7\linewidth]{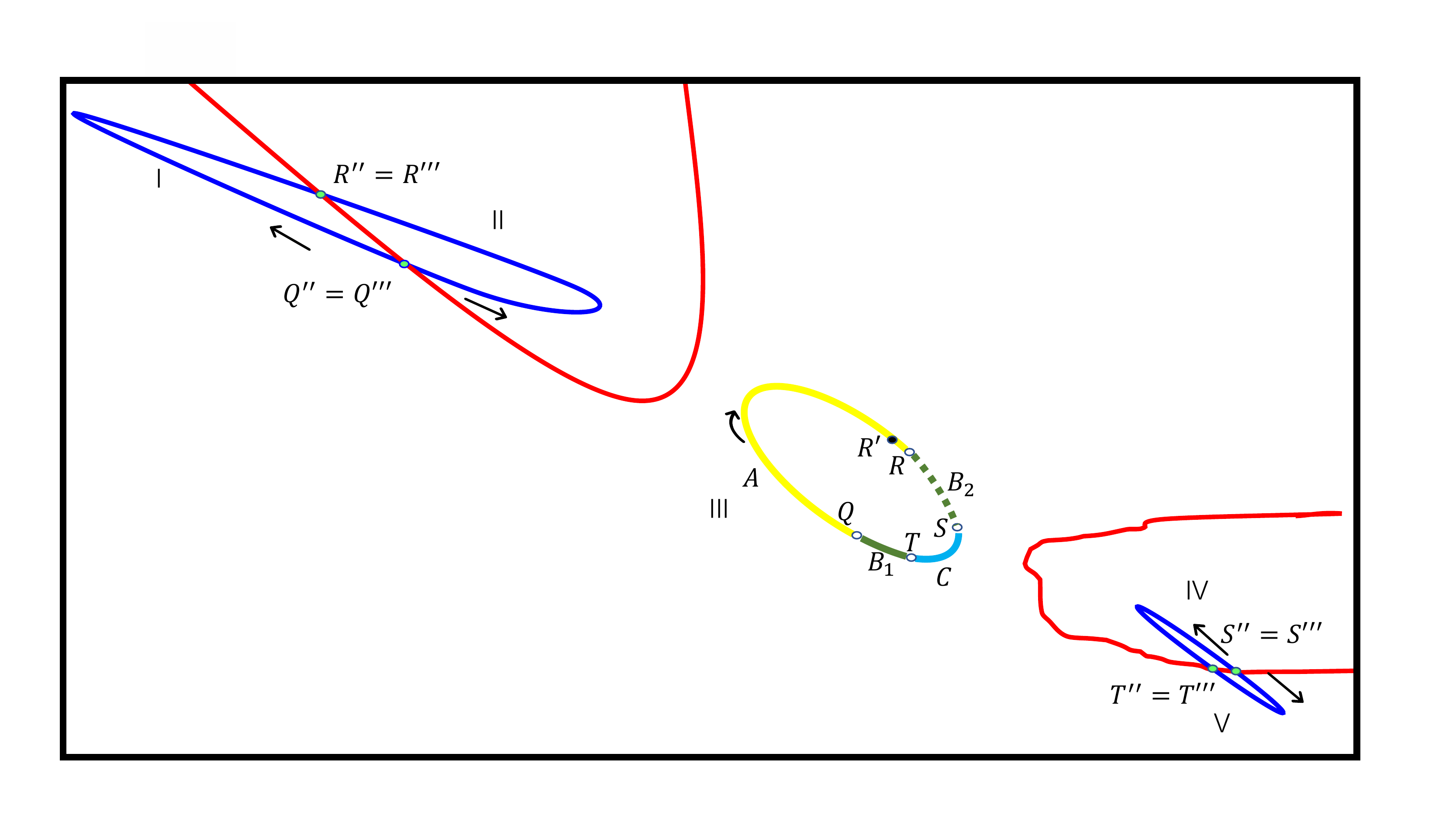}
    \caption{\small{
        \textbf{Top:} Structure of Preimages and (top inset; positive is red, and negative is blue) magnitude of the map Jacobian determinant for the Brusselator network with $b=2.1$}.
        \textbf{Bottom:} Labeling of key representative points and important regions; see text. This is a qualitative rendering of the relevant regions in the top figure, deformed to enhance visualization. 
        }
    \label{fig:smoothPreimageStructure}
\end{figure}

%
\begin{itemize}
    \item Each point in part $A$ (shown in yellow), has three inverses, located in regions \regionOne{}, \regionTwo{} and \regionThree{} respectively. The ``physically meaningful inverse", the one in \regionThree{}, is contained in the IC itself.
    \item Points in part $C$ (shown in cyan) similarly have inverses located in regions \regionThree{}, \regionFour{} and \regionFive{}. 
    \item Finally, we label two segments of the IC, located between the $A$ and $C$ segments, as $B_1$ and $B_2$ 
    (shown in dark green, with solid line and dashed lines, respectively).
    Points in these portions of the IC only have a {\em single} inverse each (that we could find within the picture): the one located {\em on the attractor itself} in region \regionThree{}.
\end{itemize}

It is informative to study the location and behavior of preimages as a phase point is moved along the invariant circle.
At the transitions from either $B_i$ part into $A$ (or $C$), two preimages (initially one preimage with multiplicity two) are born
touching the $J_0$ curve, at the junction between \regionOne{} and \regionTwo{} (or \regionFour{} and \regionFive{}).
Notice also the ``extra preimages" of the points $R$ and $Q$ ($R''$, $R'''$, $Q''$, $Q'''$) {\em off the invariant circle}, on $J_0$. The physically meaningful preimages ($R'$, $Q'$) lie on the invariant circle itself; one of them, $R'$, close to $R$, in shown in the figure.
As we move further into the $A$ (or $C$) parts of the attractor, the ``extra" two preimages separate, traverse the two blue wings of the preimage isolas, 
and then collide again on the $J_0$ curve 
as the phase point transitions from $A$ (or $C$) into the other $B_i$ part.
%

\subsection{Noninvertibility in Partially Observed Dynamic Histories}
Recall that the forward Euler discretization of the Brusselator
is a two-dimensional map
\begin{equation}  
\left\{ 
    \begin{array}{lr}  
    x_{n + 1} = x_n + \tau (a + x_n^2 y_n - (b + 1) x_n), &  \\  
    y_{n + 1} = y_n + \tau (bx_n - x_n^2 y_n). &    
    \end{array}  
\right.
\label{ode_approx_appendix}
\end{equation}

In \eqref{ode_approx_appendix}, we have two equations, but five unknowns ($x_n, y_n, x_{n + 1}, y_{n + 1}, \tau$), so the system is in principle solvable only if three of them are given. This leads to $\binom{5}{3} = 10$ possible cases, enumerated below, 
which can be thought of as generalizations of the inversion 
studied in depth in the representative paper \cite{Adomaitis1991NoninvertibilityAT, Frouzakis}.
\begin{enumerate}
    
    \item $(x_n, y_n, \tau) \Rightarrow (x_{n + 1}, y_{n + 1})$.
    \textit{(This is the usual forward dynamics case.)}
    The evolution is unique (by direct substitution into \eqref{ode_approx_appendix}).
    
    \item $(x_{n + 1}, y_{n + 1}, \tau) \Rightarrow (x_n, y_n)$. 
    \textit{(This is the case studied in depth in the paper.)}
    The backward-in-time dynamic behavior is now multi-valued. Substituting equation \eqref{yn} into the equation for $y_{n + 1}$ in system \eqref{ode_approx_appendix} we obtain
    \begin{equation}
        \tau (1 - \tau) x_n^3 + \tau (\tau a - x_{n + 1} - y_{n + 1}) x_n^2 + (\tau b + \tau - 1) x_n + (x_{n + 1} - \tau a) = 0. \label{cubic_appendix}
    \end{equation}
    \eqref{cubic_appendix} is a cubic equation w.r.t. $x_n$ if $\tau \neq 0$ and $\tau \neq 1$, which may lead to three distinct real roots, two distinct real roots (with one of them multiplicity 2), or one real root (with multiplicity 3, or with two extra complex roots). We can then substitute the solution of $x_n$ into \eqref{yn} to obtain $y_n$.

    \item $(x_n, x_{n + 1}, \tau) \Rightarrow (y_n, y_{n + 1})$. Here we know the $x$ history, and want to infer the $y$ history: create an  observer of $y$ from $x$. This is very much in the spirit of the Takens embedding theorem \cite{Takens1981}, where one uses delayed measurements of one state variable as surrogates of other, unmeasured state variables.
    For our particular Brusselator example, the $y$ dynamics inferred are unique. For the system \eqref{ode_approx_appendix}, we rearrange the equation of $x_{n + 1}$ to obtain:
    \begin{equation}
        y_n = \dfrac{x_{n + 1} - x_n + \tau (b + 1) x_n - \tau a}{\tau x_n^2}, \label{yn}
    \end{equation}
    which shows the solution for $y_n$ is unique. Substituting \eqref{yn} into \eqref{ode_approx_appendix} gives $y_{n + 1}$.

    \item $(y_n, y_{n + 1}, \tau) \Rightarrow (x_n, x_{n + 1})$. Now we use history observations for $y$ in order to infer 
    the $x$ history. The inference of the $x$ dynamic behavior is now multi-valued. From the system \eqref{ode_approx_appendix}, we can rearrange the equation of $y_{n + 1}$ to obtain
    \begin{equation}
        \tau y_n x_n^2 - \tau b x_n + (y_{n + 1} - y_n) = 0. \label{quad}
    \end{equation}
    \eqref{quad} is a quadratic equation w.r.t. $x_n$ if $\tau \neq 0$ and $y_n \neq 0$, which may lead to two distinct real roots, one real root with multiplicity 2, or two (nonphysical) complex roots. We can then substitute \eqref{quad} into \eqref{yn} to obtain $y_n$.
    
    \item $(x_n, y_{n + 1}, \tau) \Rightarrow (y_n, x_{n + 1})$. We now work with mixed, asynchronous history observations. For this particular choice of observations the inferred dynamic behavior is unique. For the system \eqref{ode_approx_appendix}, we can rearrange the equation of $y_{n + 1}$ and obtain
    \begin{equation}
        y_n = \dfrac{y_{n + 1} - \tau b x_n}{1 - \tau x_n^2}, \label{yn2}
    \end{equation}
    which shows that the solution for $y_n$ is unique. Then we can substitute \eqref{yn2} into \eqref{ode_approx_appendix} to obtain $x_{n + 1}$.
    
    \item $(y_n, x_{n + 1}, \tau) \Rightarrow (x_n, y_{n + 1})$. Interestingly, for this alternative set of asynchronous history observations, the inferred dynamic is multi-valued. From the system \eqref{ode_approx_appendix}, we can rearrange the equation of $x_{n + 1}$ and obtain
    \begin{equation}
        \tau y_n x_n^2 + (1 - \tau - \tau b) x_n + (\tau a - x_{n + 1}) = 0. \label{quad2}
    \end{equation}
    \eqref{quad2} is a quadratic equation w.r.t. $x_n$ if $\tau \neq 0$ and $y_n \neq 0$, which may lead to two distinct real roots, one real root with multiplicity 2, or two complex roots. We can then substitute \eqref{quad2} into \eqref{ode_approx_appendix} to obtain $y_{n + 1}$.
    
    \item $(x_n, y_n, x_{n + 1}) \Rightarrow (\tau, y_{n + 1})$. This is an interesting twist: several asynchronous observations, but no time label.
    Is this set of observations possible ? Does there exist a time interval $\tau$ consistent with these observations ? And how many possible $\tau$ values and possible ``history completions" exist ? 
    For this example, the inferred possible history is unique. For the system \eqref{ode_approx_appendix}, we can rearrange the equation of $x_{n + 1}$ and obtain
    \begin{equation}
        \tau = \dfrac{x_{n + 1} - x_n}{a + x_n^2 y_n - (b + 1) x_n}, \label{taux}
    \end{equation}
    which shows that the solution for $\tau$ is unique. We can then substitute \eqref{taux} into \eqref{ode_approx_appendix} to obtain $y_{n + 1}$.
    The remaining cases are alternative formulations of the same ``reconstructing history from partial observations" setting. 
    
    \item $(x_n, y_n, y_{n + 1}) \Rightarrow (\tau, x_{n + 1})$. The inferred history is again unique. For the system \eqref{ode_approx_appendix}, we can rearrange the equation of $y_{n + 1}$ and obtain
    \begin{equation}
        \tau = \dfrac{y_{n + 1} - y_n}{bx_n - x_n^2 y_n}, \label{tauy}
    \end{equation}
    which shows that the solution for $\tau$ is unique. We can then substitute \eqref{tauy} into \eqref{ode_approx_appendix} to obtain $x_{n + 1}$.
    
    \item $(y_n, x_{n + 1}, y_{n + 1}) \Rightarrow (\tau, x_n)$. The inferred history is now multi-valued. Substituting equation \eqref{taux} in the equation for $y_{n + 1}$ in system \eqref{ode_approx_appendix} we obtain:
    \begin{equation}
        y_n x_n^3 + ((y_n - y_{n + 1}) y_n - b - x_{n + 1} y_n) x_n^2 + (b x_{n + 1} + (b + 1)(y_{n + 1} - y_n)) x_n + a(y_n - y_{n + 1}) = 0. \label{cubic2}
    \end{equation}
    \eqref{cubic2} is a cubic equation w.r.t. $x_n$ if $y_n \neq 0$, which may lead to three distinct real roots, two distinct real roots (with one of them multiplicity 2), or one real root (with multiplicity 3, or with two extra complex roots). Then we could substitute the solution of $x_n$ into \eqref{taux} to obtain $\tau$.
    
    \item $(x_n, x_{n + 1}, y_{n + 1}) \Rightarrow (\tau, y_n)$. The inferred history is again multi-valued. Substituting equation \eqref{yn} in the equation for $y_{n + 1}$ in system \eqref{ode_approx_appendix} we obtain:
    \begin{equation}
        x_n^2 (a - x_n) \tau^2 + (x_n^3 - (x_{n + 1} + y_{n + 1}) x_n^2 + (b + 1) x_n -a) \tau + (x_{n + 1} - x_n) = 0. \label{quad3}
    \end{equation}
    \eqref{quad3} is a quadratic equation w.r.t. $\tau$ if $x_n \neq 0$ and $x_n \neq a$, which may lead to two distinct real roots, one real root with multiplicity 2, or two complex roots.  We can then substitute \eqref{quad3} into \eqref{yn} to obtain $y_n$. 
    
\end{enumerate}

As a demonstration,
we select the last of these cases, in which $\tau$ is an unknown,
and show that multiple consistent ``history completions", i.e. multiple roots can be found; see Table \ref{tab:historyResults}.
Roots with negative or complex $\tau$ are possible, while negative timestep could be considered as a backward-time integration, complex results have to be filtered out as nonphysical. 
The methodology and algorithms in our paper are clearly applicable in providing certifications for regions of existence of unique consistent solutions; we are currently exploring this computationally. 
\begin{table}[H]
\centering
\begin{tabular}{ccccccc}
\toprule
\multicolumn{3}{c}{Given}  & \multicolumn{4}{c}{Unknowns} \\
\cmidrule(r){1-3}
\cmidrule(r){4-7}
$x_n$ & $x_{n + 1}$ & $y_{n + 1}$ & $\tau_1$ & $\tau_2$ & $y_{n, 1}$ & $y_{n, 2}$ \\
\midrule
4.88766 & 1.62663 & 2.27734 & 0.27018 & 0.12996 & 0.06670 & -0.47845 \\
2.36082 & 3.27177 & 2.13372 & -1.51470 & 0.07929 & 0.98342 & 3.15257 \\
2.19914 & 1.97336 & 3.22943 & -1.51394 & -0.02572 & 1.18823 & 2.97282 \\
4.60127 & 2.27780 & 2.21088 & \multicolumn{2}{c}{(0.09960 $\pm$ 0.14337$i$)} & \multicolumn{2}{c}{(0.24609 $\pm$ 0.51630$i$)} \\
\bottomrule
\end{tabular}
\caption{$(x_n, x_{n + 1}, y_{n + 1}) \Rightarrow (\tau, y_n)$, where $a = 1$, $b = 2$.}
\label{tab:historyResults}
\end{table}

\subsection{Extensions to Residual Architectures}


We demonstrate that our algorithms are also applicable to residual networks with ReLU activations. 
The MILP method does not extend in a simple way
to networks with $\tanh$ or sigmoid activation,
but we show here that simple algebraic formulas,
like the residual connection, 
are addressable in this framework.
%
This follows from the fact that the identity function
is equivalent to a ReLU multi-layer perceptron (MLP) with an arbitrary number of hidden layers,
\begin{equation}
    x = g(x) - g(-x) = g(g(x)) - g(g(-x)) = g(g(g(x))) - g(g(g(-x))) = \cdots , \label{eq: relu_id}
\end{equation}
where $g(x) = \max(0,x)$ is the ReLU function. 
Because ReLU is idempotent $g(g(x)) = g(x)$, we are able to add more and more nested versions in the right side of \eqref{eq: relu_id}.
Thus one could transform a ReLU ResNet with fully-connected layers 
to a single ReLU MLP by applying the equivalence \eqref{eq: relu_id}.

\begin{proposition} \label{theorem: ResNet}
    A ReLU ResNet with $\ell$ fully-connected layers in its residual architecture is equivalent to an MLP with the same number of layers.
\end{proposition}
%

\begin{proof}
    Suppose $f: \mathbb{R}^m \mapsto \mathbb{R}^m$ is a ResNet. 
    We can rewrite $y=f(x)$ as
    \begin{equation}
    \begin{aligned}
        y = & W^{(\ell)} g(W^{(\ell-1)} g(\cdots g(W^{(0)} x + b^{(0)}) + \cdots ) + b^{(\ell-1)}) + b^{(\ell)} + x \\
        = & W^{(\ell)} g(W^{(\ell-1)} g(\cdots g(W^{(0)} x + b^{(0)}) + \cdots ) + b^{(\ell-1)}) + b^{(\ell)}  \\
        & + I_m g(I_m g(\cdots g(I_m x) + \cdots )) + (-I_m) g(I_m g(\cdots g(-I_m x) + \cdots )) \\
        = & (W^{(\ell)}, I_m, -I_m) g(
            \begin{bmatrix}
            W^{(\ell-1)} & & \\
            & I_m & \\
            & & -I_m \\
            \end{bmatrix} g(\cdots g(\begin{bmatrix}
            W^{(0)} \\
            I_m \\
            -I_m \\
            \end{bmatrix} x + \begin{bmatrix}
            b^{(0)} \\
            0_m \\
            0_m \\
            \end{bmatrix}) + \cdots ) \\
            & + \begin{bmatrix}
            b^{(\ell-1)} \\
            0_m \\
            0_m \\
            \end{bmatrix}
        ) + b^{(l)}. \label{eq: res}
    \end{aligned}
    \end{equation}
    Here, $I_m \in \mathbb{R}^{m \times m}$ is an identity matrix, and $0_m \in \mathbb{R}^m$ is a zero vector. If we denote
    \begin{equation}
    \begin{aligned}
        W'^{(0)} &= \begin{bmatrix}
        W^{(0)} \\
        I_m \\
        -I_m \\
        \end{bmatrix}, W'^{(\ell)} = (W^{(\ell)}, I_m, -I_m), \\
        W'^{(j)} &= \begin{bmatrix}
        W^{(j)} & & \\
        & I_m & \\
        & & -I_m \\
        \end{bmatrix} \text{ for } j = 1, 2, \cdots, \ell - 1, \\
        b'^{(k)} &= \begin{bmatrix}
        b^{(k)} \\
        0_m \\
        0_m \\
        \end{bmatrix} \text{ for } k = 0, 1, \cdots, \ell - 1, \text{ and } b'^{(\ell)} = b^{(\ell)}, 
    \end{aligned}
    \label{eq: res_sub}
    \end{equation}
    then the function
    \begin{equation}
        y = W'^{(\ell)} g(W'^{(\ell-1)} g(\cdots g(W'^{(0)} x + b'^{(0)}) + \cdots ) + b'^{(\ell-1)}) + b'^{(\ell)} \label{eq: res_mlp}
    \end{equation}
    is a ReLU MLP with $\ell$ layers.
\end{proof}

{\bf Structurally Invertible Networks.}
It is interesting to consider how our algorithm would perform when the network under study is invertible by architectural construction (e.g. an  invertible ResNet (``i-ResNet'', \cite{behrmann2019invertible}).
Then there is only the trivial solution to the MILP in \eqref{eq: local invertibility MIP} for any $r>0$ (two identical points). What we can do in such cases is to request a certificate of guarantee that we are sufficiently far from noninvertibility boundaries -- e.g. by a threshold larger than, say,  $10^6$.  
This is suggestive of global invertibility of the i-ResNet, 
and serves as a sanity check of the algorithm.

%

%
{\bf Computational Effort}.
In general, several key factors impact the computational time of the MILP: the input dimension $n_0$, the number of layers $\ell$, the total number of neurons $\sum_{i=1}^{\ell} n_i$, and the radius parameter $r$.
Because a multi-layer network can be approximated to desired accuracy
by a single-layer network with enough neurons, 
we will perform our experiment with a single-layer perceptron 
($\ell = 1$)
and observe the dependence of the running time on $n_0, n_1$ and $r$ by optimizing starting from multiple randomly-generated i-ResNets. 
To reduce the influence of difficult i-ResNet parameters that might cause the optimizer to stall, diverge, converge very slowly, 
or (of most concern) halt by our 30-minute timeout,
we track the median of the running times for replicate experiments.
See Figure \ref{fig: rt} for these results.
\begin{figure}[H]
	\centering
	\includegraphics[width=1\linewidth]{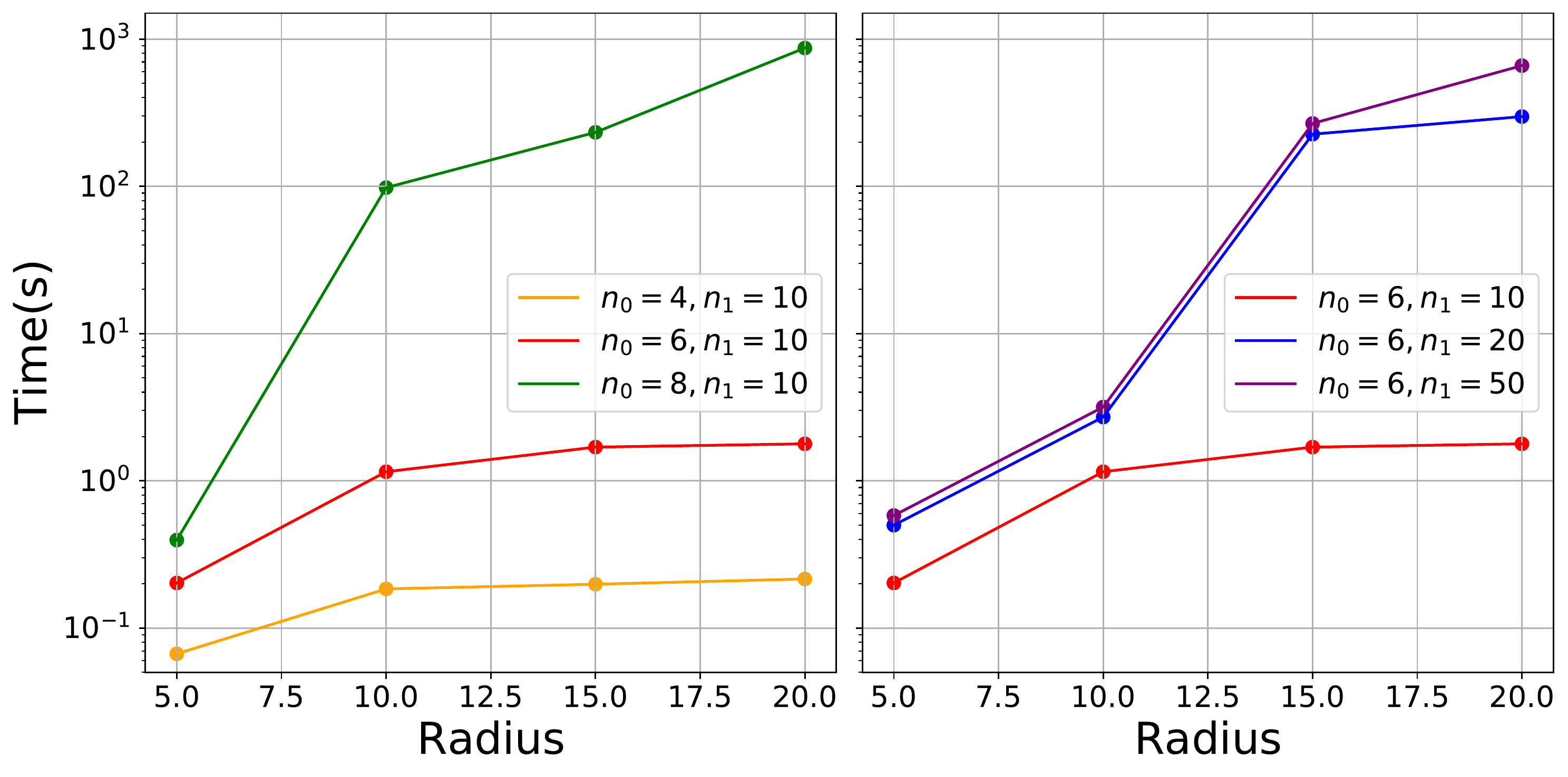}
	\caption{\small{Running time of the algorithm on a single-layer invertible ResNet as the network size varies.}}
	\label{fig: rt}
\end{figure}
\noindent
We observe that the $n_0$ hyperparameter
has a greater impact on the running time than the $n_1$ hyperparameter. 

\section{Proof of Theorems and Corollaries}


\subsection{Proposition Regarding Solutions to Problem \ref{problem 1} and Problem \ref{problem 2}}

\begin{proposition*} \label{cor radius_appendix}
    For a given function $f: \mathbb{R}^m \mapsto \mathbb{R}^m$ and a point $x_c \in \mathbb{R}^m$, if $r$ and $R$ are optimal solutions to problems \ref{problem 1} and \ref{problem 2} respectively, then we must have $r \leq R$.
\end{proposition*}

Consider a point $x \in \mathcal{B}_q(x_c,r) \setminus \{x_c \}$. Since $f$ is invertible on $\mathcal{B}_q(x_c,r)$, 
we must have $f(x') \neq f(x)$ for all $x' \in \mathcal{B}_q(x_c,r) \setminus \{x \}$. 
In particular, by choosing $x=x_c$, we have $f(x') \neq f(x_c)$ for all $x' \in \mathcal{B}_q(x_c,r) \setminus \{x \}$. Thus, we must have $r \leq R$.

\subsection{Proof of Theorem \ref{theorem: local_inv}}

\begin{theorem*} \label{theorem: local_inv_appendix}
	Let $f \colon \mathbb{R}^m \to \mathbb{R}^m$ be a continuous function and $\mathcal{B} \subset \mathbb{R}^m$ be a compact set. Consider the following optimization problem,
\begin{alignat}{2}
p^\star \leftarrow& \mathrm{max} \quad  && \|x-y\| \quad \text{subject to } x,y \in \mathcal{B}, \quad f(x)=f(y).
\end{alignat}
Then $f$ is invertible on $\mathcal{B}$ if and only if $p^\star =0$.
\end{theorem*}

Suppose $f$ is invertible on $\mathcal{B}$. Then for all $x,y \in \mathcal{B}$ for which $f(x)=f(y)$, we must have $x = y$. Therefore, the objective function for Problem 1 is zero on the feasible set. Hence, $p^\star=0$. Conversely, suppose $p^\star=0$. Then $x=y$ for all $x,y \in \mathcal{B}$ such that $f(x)=f(y)$, hence invertibility.

\subsection{Proof of Theorem \ref{theorem: local_pseudo}}

\begin{theorem*}  \label{theorem: local_pseudo_appendix}
	Let $f \colon \mathbb{R}^m \to \mathbb{R}^m$ be a continuous function and $\mathcal{B} \subset \mathbb{R}^m$ be a compact set. Suppose $x_c \in \mathcal{B}$. Consider the following optimization problem,
	\begin{align}
		P^\star \leftarrow  \mathrm{max} \quad  \|x-x_c\| \quad
		\text{subject to } x \in \mathcal{B},  \quad f(x)=f(x_c).
	\end{align}
	Then we have $f(x) \neq f(x_c)$ for all $x \in \mathcal{B} \setminus \{x_c\}$ if and only if $P^\star =0$.
\end{theorem*}

Suppose $f(x) \neq f(x_c)$ for all $x \in \mathcal{B} \setminus \{x_c\}$. 
Then, the only feasible point in the  optimization of Problem 2 is $x = x_c$. 
Hence, $P^\star=0$. 
Conversely, start by assuming $P^\star=0$.
Suppose there exists a $x' \in \mathcal{B} \setminus \{x_c\}$ such that $f(x')=f(x_c)$. Then, we must have $0 < \|x'-x_c\| \ \leq P^\star=0$, which is a contradiction. Therefore, we must have $f(x) \neq f(x_c)$ for all $x \in \mathcal{B} \setminus \{x_c\}$.

\subsection{Proof of Theorem \ref{theorem: trans}}

\begin{theorem*}  \label{theorem: trans_appendix}
	Let $f_1 \colon \mathbb{R}^m \to \mathbb{R}^n$, $f_2 \colon \mathbb{R}^m \to \mathbb{R}^n$ be two continuous functions and $\mathcal{B} \subset \mathbb{R}^m$ be a compact set. 
	Consider the following optimization problem,
    \begin{align} \label{opt problem 3: appendix}
		p_{12}^\star \leftarrow  \mathrm{max} \quad  \|f_2(x^{(1)}) - f_2(x^{(2)})\| \quad
		\text{subject to } x^{(1)}, x^{(2)} \in \mathcal{B},  \quad f_1(x^{(1)}) = f_1(x^{(2)}).
	\end{align}
	Then \textbf{(a)} $f_2$ is a function of $f_1$ on $\mathcal{B}$ if and only if \textbf{(b)} $p_{12}^\star = 0$.
\end{theorem*}

We first set up a definition (with a slight abuse of notation) of preimage set to simplify our proof.

\begin{definition}
    For a given function $f \colon \mathcal{X} \mapsto \mathcal{Y}$, $\mathcal{X} \subseteq \mathbb{R}^m$, $\mathcal{Y} \subseteq \mathbb{R}^n$, the preimage of $y \in \mathcal{Y}$ is $f^{-1}(y) = \{x \in \mathcal{X} \mid f(x)=y\}$.
    %
\end{definition}

We then prove the following theorem.

\begin{theorem}
    For two functions $f_i \colon \mathcal{X} \mapsto \mathcal{Y}_i$, $\mathcal{X} \subseteq \mathbb{R}^m$, $\mathcal{Y}_i \subseteq \mathbb{R}^n$, $i=1,2$, we have \textbf{(a)} output of $f_2$ is a function of output of $f_1$ if and only if \textbf{(c)} output of $f_2$ is constant over the preimage set $f_1^{-1}(y_1)$ for all $y_1 \in \mathcal{Y}_1$.
\end{theorem}

\begin{proof}
    We will show the equivalence of \textbf{(a)} and \textbf{(c)}.
    
    \textbf{(c)} $\Rightarrow$ \textbf{(a)}: If $f_1^{-1}(y_1)$ is a singleton $\{x_1\}$, then $f_2(x_1)=y_2 \in \mathcal{Y}_2$ is the only value corresponding to $y_1$. Otherwise, we could arbitrarily choose two different values $x_{A}, x_{B} \in f_1^{-1}(y_1)$, and we must have $f_2(x_A) = f_2(x_B)=y_2 \in \mathcal{Y}_2$. Therefore, we can find a unique $y_2 \in \mathcal{Y}_2$ that corresponds to the given $y_1$, which infers the existence of a mapping from $\mathcal{Y}_1$ to $\mathcal{Y}_2$.
    
    \textbf{(a)} $\Rightarrow$ \textbf{(c)}: We prove this by contradiction. Suppose $f_2$ is a function of output of $f_1$, and $\exists y_1 \in \mathcal{Y}_1$ and $\exists x_A, x_B \in f_1^{-1}(y_1)$ such that $f_2(x_A) \neq f_2(x_B)$ (i.e. $f_2$ is constant over $f_1^{-1}(y_1)$). Therefore, we can find a $y_1 \in \mathcal{Y}_1$ simultaneously corresponding to two different values $f_2(x_A)$ and $f_2(x_B)$ in $\mathcal{Y}_2$, showing the contradiction with \textbf{(a)}.
\end{proof}

It is not hard to show \textbf{(b)} ``$p_{12}^* = 0$'' in \eqref{opt problem 3: appendix} is equivalent with the statement that $f_2(x)$ is constant for $\forall x \in f_1^{-1}(f_1(x))$, which is just rephrasing of \textbf{(c)} by denoting $y_1 = f_1(x)$, and therefore, we show the equivalence of \textbf{(a)} and \textbf{(b)}.

\end{document}